\documentclass[11pt]{article}

\usepackage{natbib}
\usepackage[utf8]{inputenc} 
\usepackage[T1]{fontenc}    
\usepackage{hyperref}       
\usepackage{url}            
\usepackage{booktabs}       
\usepackage{amsfonts}       
\usepackage{nicefrac}       
\usepackage{microtype}      
\usepackage{xcolor}         
\usepackage{graphicx}
\usepackage{amsmath, amssymb, amsthm}
\usepackage{mathtools}
\usepackage{algorithm}
\usepackage{algorithmic}
\usepackage{subcaption}
\usepackage{thmtools}

\usepackage{fullpage}
\usepackage{authblk}
\usepackage{tikz}
\usepackage{bbding}
\usepackage{thm-restate}

\newtheorem{property}{Property}
\newtheorem{assumption}{Assumption}

\usepackage{tikz}
\usetikzlibrary{
    decorations.pathmorphing, 
    arrows.meta,             
    positioning,             
    calc
}
\usetikzlibrary{shapes, arrows.meta, positioning, fit, backgrounds, decorations.pathreplacing, calc, shadows}

\definecolor{patientblue}{RGB}{70, 130, 180}
\definecolor{donorred}{RGB}{220, 80, 80}
\definecolor{clustergray}{RGB}{240, 240, 245}
\definecolor{edgegreen}{RGB}{46, 139, 87}

\tikzset{
    complexheart/.pic={
        \begin{scope}[yscale=-1, shift={(-4.5, -9)}]
            \draw[fill=red!30!white](.456,3.236)
                .. controls (.422,4.168) and (.408,5.095) .. (.461,6.046)
                .. controls (.475,6.228) and (.365,6.400) .. (.379,6.601)
                -- (.819,11.816)
                .. controls (.843,12.194) and (.838,13.389) .. (.900,13.972)
                .. controls (.943,14.340) and (2.870,14.340) .. (2.903,13.972)
                .. controls (2.927,13.699) and (2.932,13.436) .. (2.903,13.169)
                -- (1.847,5.401)
                .. controls (1.914,4.627) and (2.033,3.924) .. (2.129,3.193)
                .. controls (2.177,2.906) and (.441,2.863) .. (.456,3.236) ;
            \draw[fill=red!30!white] (.456,3.236)
                .. controls (.422,4.168) and (.408,5.095) .. (.461,6.046)
                .. controls (.475,6.228) and (.365,6.400) .. (.379,6.601)
                -- (.819,11.816)
                .. controls (.843,12.194) and (.838,13.389) .. (.900,13.972)
                .. controls (.943,14.340) and (2.870,14.340) .. (2.903,13.972)
                .. controls (2.927,13.699) and (2.932,13.436) .. (2.903,13.169)
                -- (1.847,5.401)
                .. controls (1.914,4.627) and (2.033,3.924) .. (2.129,3.193)
                .. controls (2.177,2.906) and (.441,2.863) .. (.456,3.236) ;
            \draw[fill=red!60!white] (.303,6.697)
                .. controls (-.161,7.442) and (-.022,9.641) .. (.484,11.118)
                .. controls (.795,12.022) and (1.297,12.414) .. (2.062,12.801)
                .. controls (3.429,13.513) and (5.112,13.714) .. (6.598,13.795)
                .. controls (9.495,13.957) and (9.738,11.735) .. (8.974,8.967)
                .. controls (8.744,8.126) and (8.123,7.815) .. (7.917,7.222)
                .. controls (7.683,6.606) and (7.310,5.999) .. (6.598,5.420)
                .. controls (5.570,4.579) and (1.230,5.210) .. (.303,6.697) ;
            \fill[fill=red!25!white] (2.540,6.405)
                .. controls (2.296,6.257) and (1.890,6.300) .. (1.608,6.424)
                .. controls (1.274,6.577) and (1.101,6.802) .. (1.001,7.151)
                .. controls (.858,7.652) and (1.025,8.293) .. (1.106,8.795)
                .. controls (1.149,9.072) and (1.149,9.402) .. (1.254,9.684)
                .. controls (1.278,9.746) and (1.278,9.823) .. (1.360,9.870)
                .. controls (1.426,9.909) and (1.412,9.732) .. (1.460,9.651)
                .. controls (1.747,9.144) and (2.210,8.346) .. (2.502,7.834)
                .. controls (2.602,7.657) and (2.631,7.437) .. (2.531,7.303)
                .. controls (2.411,7.146) and (2.206,6.883) .. (2.306,6.768)
                .. controls (2.382,6.682) and (2.621,6.453) .. (2.540,6.405) ;
            \fill[fill=red!25!white] (4.204,6.606)
                .. controls (5.413,6.434) and (5.327,7.543) .. (5.466,7.796)
                .. controls (5.599,8.054) and (6.350,8.785) .. (6.498,9.029)
                .. controls (6.646,9.273) and (7.282,9.947) .. (7.067,10.616)
                .. controls (6.847,11.290) and (5.776,11.582) .. (5.212,11.429)
                .. controls (4.643,11.271) and (3.773,10.941) .. (3.166,11.113)
                .. controls (2.559,11.280) and (1.350,10.793) .. (1.790,10.052)
                .. controls (2.225,9.311) and (3.640,6.692) .. (4.204,6.606) ;
            \draw[fill=yellow] (1.307,3.107)
                .. controls (1.675,3.107) and (1.976,3.169) .. (1.976,3.250)
                .. controls (1.976,3.327) and (1.675,3.394) .. (1.307,3.394)
                .. controls (.934,3.394) and (.633,3.327) .. (.633,3.250)
                .. controls (.633,3.169) and (.934,3.107) .. (1.307,3.107) ;
            \draw[fill=blue!20!red!80!white] (1.307,3.107)
                .. controls (1.675,3.107) and (1.976,3.169) .. (1.976,3.250)
                .. controls (1.976,3.327) and (1.675,3.394) .. (1.307,3.394)
                .. controls (.934,3.394) and (.633,3.327) .. (.633,3.250)
                .. controls (.633,3.169) and (.934,3.107) .. (1.307,3.107) ;
            \fill[fill=red!10!white] (1.024,14.000)
                .. controls (.991,13.929) and (.924,12.146) .. (1.039,12.323)
                .. controls (1.120,12.452) and (1.498,12.753) .. (1.870,12.767)
                .. controls (2.004,12.777) and (1.918,14.129) .. (1.875,14.163)
                .. controls (1.789,14.230) and (1.110,14.173) .. (1.024,14.000) ;
            \fill[fill=red!10!white] (.618,3.432)
                -- (.618,6.247)
                .. controls (.852,5.941) and (1.086,5.745) .. (1.345,5.578)
                .. controls (1.364,5.033) and (1.368,4.125) .. (1.345,3.556)
                .. controls (1.096,3.556) and (.862,3.518) .. (.618,3.432) ;
            \draw[fill=white] (4.872,5.215)
                .. controls (4.944,5.607) and (4.891,5.884) .. (4.595,6.061)
                .. controls (4.389,6.185) and (3.749,6.013) .. (3.864,5.664)
                .. controls (3.974,5.339) and (4.088,5.081) .. (4.346,4.923)
                .. controls (4.533,4.813) and (4.858,5.119) .. (4.872,5.215) ;
            
            \draw[fill=red!80!blue] (5.662,5.665)
                .. controls (5.643,5.732) and (5.624,5.803) .. (5.604,5.880)
                .. controls (5.944,6.076) and (6.355,6.640) .. (6.522,7.314)
                .. controls (6.699,8.031) and (6.336,8.284) .. (5.853,8.796)
                .. controls (5.523,9.154) and (5.313,9.541) .. (5.303,10.014)
                .. controls (5.485,9.106) and (5.844,8.958) .. (6.030,8.877)
                .. controls (6.278,8.767) and (6.221,9.221) .. (6.159,10.029)
                .. controls (6.341,9.412) and (6.298,8.944) .. (6.274,8.752)
                .. controls (6.245,8.561) and (6.589,8.370) .. (6.685,8.093)
                .. controls (6.799,7.777) and (6.924,8.069) .. (6.991,8.208)
                .. controls (7.153,8.547) and (7.311,8.681) .. (7.263,8.987)
                .. controls (7.096,10.033) and (6.780,10.387) .. (6.250,10.650)
                .. controls (6.728,10.492) and (6.833,10.373) .. (6.976,10.230)
                .. controls (6.991,10.440) and (6.900,10.870) .. (6.962,11.209)
                .. controls (6.957,10.669) and (7.277,9.197) .. (7.469,9.044)
                .. controls (7.579,8.953) and (8.296,9.632) .. (8.372,10.139)
                .. controls (8.458,10.698) and (8.936,11.535) .. (8.716,12.213)
                .. controls (8.970,11.458) and (8.535,10.636) .. (8.635,10.445)
                .. controls (8.778,10.641) and (9.060,11.037) .. (9.027,11.257)
                .. controls (8.994,11.477) and (9.170,11.874) .. (9.094,12.333)
                .. controls (9.228,11.883) and (9.065,11.702) .. (9.118,11.377)
                .. controls (9.166,11.052) and (8.979,10.626) .. (8.836,10.359)
                .. controls (8.549,9.828) and (7.889,9.187) .. (7.373,8.595)
                .. controls (7.708,8.796) and (8.535,9.077) .. (8.487,9.369)
                .. controls (8.434,9.661) and (8.860,9.900) .. (9.027,10.306)
                .. controls (8.903,9.928) and (8.587,9.656) .. (8.568,9.321)
                .. controls (8.554,8.982) and (8.023,8.815) .. (7.889,8.762)
                .. controls (7.995,8.752) and (8.253,8.748) .. (8.348,8.834)
                .. controls (8.444,8.920) and (8.702,9.097) .. (8.912,9.044)
                .. controls (8.539,9.044) and (8.463,8.671) .. (8.286,8.657)
                .. controls (8.109,8.647) and (7.641,8.566) .. (7.531,8.470)
                .. controls (7.015,8.016) and (6.599,7.223) .. (6.819,7.271)
                .. controls (7.344,7.381) and (7.899,7.471) .. (8.377,7.988)
                .. controls (8.023,7.567) and (7.999,7.562) .. (7.478,7.314)
                .. controls (7.091,7.137) and (6.537,6.979) .. (6.355,6.511)
                .. controls (6.250,6.229) and (5.887,5.789) .. (5.662,5.665) ;
            \draw[fill=red!80!blue] (3.387,5.951)
                .. controls (3.779,7.003) and (2.665,8.600) .. (1.857,9.866)
                .. controls (1.384,10.607) and (1.704,11.157) .. (2.364,11.305)
                .. controls (2.842,11.410) and (4.515,11.004) .. (4.993,11.023)
                .. controls (5.480,11.037) and (5.604,10.856) .. (5.882,10.559)
                .. controls (5.815,10.894) and (5.256,11.305) .. (4.997,11.209)
                .. controls (4.735,11.114) and (4.199,11.229) .. (3.975,11.338)
                .. controls (4.280,11.425) and (4.787,11.401) .. (5.179,11.601)
                .. controls (5.571,11.802) and (5.920,11.750) .. (6.159,11.300)
                .. controls (6.140,11.730) and (5.538,11.955) .. (5.126,11.783)
                .. controls (4.715,11.616) and (3.736,11.434) .. (3.401,11.434)
                .. controls (3.176,11.429) and (3.009,11.468) .. (2.617,11.563)
                .. controls (3.324,11.635) and (4.252,12.084) .. (5.165,12.304)
                .. controls (5.863,12.471) and (6.680,12.452) .. (6.790,12.189)
                .. controls (6.895,11.946) and (7.072,11.840) .. (7.473,11.793)
                .. controls (7.000,12.008) and (6.866,12.156) .. (6.814,12.438)
                .. controls (7.015,12.490) and (7.411,12.256) .. (7.626,12.457)
                .. controls (7.846,12.658) and (8.439,12.653) .. (8.864,12.381)
                .. controls (8.386,12.730) and (7.770,12.710) .. (7.607,12.557)
                .. controls (7.445,12.405) and (7.077,12.553) .. (6.862,12.591)
                .. controls (6.647,12.634) and (6.446,12.644) .. (5.891,12.548)
                .. controls (5.714,12.620) and (5.657,12.715) .. (5.867,12.792)
                .. controls (6.082,12.835) and (6.537,12.806) .. (7.014,12.806)
                .. controls (7.258,12.806) and (7.430,12.935) .. (7.631,12.811)
                .. controls (7.832,12.682) and (8.033,12.873) .. (8.324,13.016)
                .. controls (7.999,12.906) and (7.746,12.768) .. (7.660,12.897)
                .. controls (7.574,13.031) and (7.263,12.916) .. (7.163,12.940)
                .. controls (7.292,12.992) and (7.493,12.935) .. (7.507,13.136)
                .. controls (7.526,13.337) and (8.057,13.284) .. (7.999,13.652)
                .. controls (7.923,13.294) and (7.540,13.423) .. (7.406,13.179)
                .. controls (7.277,12.940) and (6.885,13.055) .. (6.680,13.031)
                .. controls (6.250,12.973) and (5.557,13.074) .. (5.461,12.639)
                .. controls (5.217,12.600) and (4.921,12.514) .. (4.682,12.419)
                .. controls (4.687,12.538) and (5.103,13.026) .. (5.461,13.102)
                .. controls (5.815,13.174) and (6.532,13.155) .. (6.699,13.681)
                .. controls (6.345,13.174) and (5.614,13.356) .. (5.294,13.260)
                .. controls (4.978,13.160) and (4.213,12.285) .. (3.898,12.189)
                .. controls (3.583,12.099) and (3.272,11.922) .. (3.157,11.989)
                .. controls (3.062,12.084) and (3.119,12.600) .. (3.583,12.581)
                .. controls (4.046,12.562) and (4.362,12.882) .. (4.620,13.356)
                .. controls (4.271,12.897) and (3.721,12.615) .. (3.358,12.744)
                .. controls (2.990,12.868) and (2.823,11.936) .. (2.598,11.812)
                .. controls (2.373,11.683) and (1.594,11.377) .. (1.503,11.042)
                .. controls (1.465,11.377) and (2.106,11.903) .. (2.043,12.677)
                .. controls (1.953,11.922) and (1.522,11.783) .. (1.317,11.377)
                .. controls (1.111,10.966) and (1.212,10.555) .. (1.417,9.991)
                .. controls (1.422,9.971) and (3.836,6.267) .. (2.588,5.985)
                .. controls (2.235,5.904) and (1.981,5.813) .. (1.757,5.660)
                .. controls (1.871,4.489) and (2.187,2.481) .. (2.650,2.405)
                .. controls (3.176,2.314) and (4.185,2.878) .. (4.175,3.289)
                .. controls (4.166,3.724) and (3.817,4.766) .. (3.994,5.201)
                .. controls (3.759,5.397) and (3.277,5.617) .. (3.387,5.951) ;
            \fill[fill=red!90!blue!70!white] (2.942,5.378)
                .. controls (2.875,5.339) and (2.736,5.378) .. (2.722,5.459)
                .. controls (2.712,5.535) and (2.741,5.583) .. (2.779,5.674)
                .. controls (3.013,6.066) and (3.200,6.716) .. (2.980,7.266)
                .. controls (2.803,7.720) and (2.564,8.226) .. (2.296,8.661)
                .. controls (2.048,9.068) and (1.627,9.670) .. (1.565,10.363)
                .. controls (1.579,10.105) and (1.837,9.737) .. (1.952,9.512)
                .. controls (2.263,8.905) and (3.071,7.815) .. (3.214,7.180)
                .. controls (3.362,6.534) and (3.243,5.875) .. (3.071,5.597)
                .. controls (3.066,5.540) and (3.028,5.425) .. (2.942,5.378) ;
            \fill[fill=red!90!blue!70!white] (2.793,2.424)
                .. controls (2.253,3.542) and (2.115,3.948) .. (2.014,5.712)
                .. controls (2.019,5.784) and (2.349,5.903) .. (2.339,5.841)
                .. controls (2.301,4.665) and (2.650,3.289) .. (3.138,2.486)
                .. controls (3.032,2.447) and (2.870,2.419) .. (2.793,2.424) ;
            \draw (2.526,6.104)
                .. controls (2.990,5.937) and (2.541,6.688) .. (2.545,6.673)
                .. controls (2.454,6.831) and (2.521,6.869) .. (2.564,6.922)
                .. controls (2.717,7.123) and (2.932,7.165) .. (2.660,7.471) ;
            \draw (.461,3.241)
                .. controls (.576,3.657) and (2.187,3.657) .. (2.129,3.198) ;
            \draw[fill=red!30!white] (3.687,5.731)
                .. controls (3.520,6.229) and (5.541,6.616) .. (5.556,6.214)
                .. controls (5.570,5.827) and (5.790,5.325) .. (5.800,4.919)
                .. controls (5.809,4.584) and (4.480,4.250) .. (4.308,4.517)
                .. controls (4.084,4.866) and (3.854,5.249) .. (3.687,5.731) ;
            \draw (4.299,4.531)
                .. controls (4.189,4.885) and (5.752,5.315) .. (5.800,4.961) ;
            \draw[fill=blue!20!red!80!white] (5.097,4.594)
                .. controls (5.436,4.694) and (5.699,4.823) .. (5.680,4.919)
                .. controls (5.656,5.009) and (5.360,5.000) .. (5.020,4.899)
                .. controls (4.681,4.799) and (4.428,4.675) .. (4.452,4.584)
                .. controls (4.475,4.493) and (4.757,4.493) .. (5.097,4.594) ;
        \end{scope}
    }
}

\newcommand{\E}{\mathbb{E}}
\newcommand{\I}{\mathbb{I}}
\newcommand{\Loss}{\mathcal{L}}
\newcommand{\pr}{\mathbb{P}}

\newcommand{\opt}{\mathrm{OPT}}
\newcommand{\alg}{\mathsf{ALG}}

\newcommand{\rel}{\hat{T}}
\newcommand{\dcg}{\text{DCG}}
\newcommand{\ndcg}{\text{NDCG}}
\newcommand{\idcg}{\text{IDCG}}
\newcommand{\dcgk}[1]{\dcg @{#1}}
\newcommand{\ndcgk}[1]{\ndcg @{#1}}
\newcommand{\idcgk}[1]{\idcg @{#1}}

\newcommand{\ipcw}{\text{IPCW}}
\newcommand{\ey}{\text{EY}}
\newcommand{\restrict}{(\tau)}
\newcommand{\pwg}{\mathsf{PWA}}

\usepackage{cleveref}
\Crefname{property}{Property}{Properties}
\Crefname{assumption}{Assumption}{Assumption}

\title{Aligning Data-Driven Predictors with Allocation: \\
A Decision-Focused Approach to Survival Analysis}

\author[1]{Itai Zilberstein\thanks{Correspondence to \texttt{izilbers@cs.cmu.edu}}}
\author[1]{Ioannis Anagnostides}
\author[1,2]{Tuomas Sandholm}

\affil[1]{Department of Computer Science, Carnegie Mellon University, Pittsburgh, PA}
\affil[2]{\small{Additional affiliations: Strategy Robot, Inc., Strategic Machine, Inc., Optimized Markets, Inc.}}

\begin{document}

\maketitle
\thispagestyle{empty}

\begin{abstract}
  Machine learning predictors have become essential tools for guiding automated decision making. However, a major misalignment persists: predictive models are typically optimized in terms of standard statistical metrics in isolation from the algorithmic tasks they inform.  We highlight this incongruity in the high-stakes domain of organ allocation by demonstrating that any algorithm relying on (even highly accurate) survival predictors optimized for standard metrics---such as the \textit{Concordance index (C-index)}---can yield arbitrarily poor outcomes when used for allocation, failing to guarantee utility better than a uniform random selection.
  To bridge the gap between survival analysis and policy optimization, we introduce a \emph{decision-focused learning} approach based on optimizing \emph{normalized discounted cumulative gain (NDCG)}, a mainstay metric in information retrieval. We establish the utility of NDCG in survival analysis by proving that it translates to guarantees on the performance of allocation. Empirically, we propose a bootstrapping approach to optimize the NDCG of existing survival models. Unlike prior work, we also address the challenge of right censorship when evaluating ranking. On historical heart transplant data from the US, our method dramatically boosts the NDCG of baseline models by 50-100\%, which translates to tens of thousands of additional life years gained annually when deployed for transplant allocation. We anticipate that our framework will find broader applications in decision making with predictions.
\end{abstract}
 
\newpage
\setcounter{page}{1}

\section{Introduction}

Real-world decision making increasingly relies on algorithms powered by \textit{machine learning (ML)} predictors trained on vast amounts of historical data. From resource allocation to automated planning and scheduling, these data-driven systems are deployed in high-stakes environments. However, an underlying disconnect persists: the development of classical algorithms for these problems is often disjoint from the design of the predictive models they leverage. ML models are typically optimized in isolation for standard statistical metrics, while the downstream algorithms using these predictions either fail to account for the predictor's performance profile or suffer because the algorithmic objective is misaligned with the model's training. The gap between predictive accuracy and algorithmic utility can lead to catastrophic outcomes, particularly in high-stakes applications such as organ allocation.

Organ transplantation is the treatment of choice for many terminal illnesses. Across organ types, the demand for deceased-donor organs outpaces the available supply~\citep{Cameli22:Donor}. In the US alone, thousands of patients with end-stage heart failure are waitlisted for a life-saving organ. 

The current US heart transplant allocation policy places patients into rigid hierarchical tiers and allocates the organ to the highest-priority compatible patient. The policy often treats patients with heterogeneous clinical profiles as effectively identical. A major criticism of the policy is that it does not leverage finer-grained predictions of pretransplant mortality and post-transplant outcomes~\citep{Shore20:Changes,Zhang24:Development}. As a result, the US is transitioning to new data-driven solutions to improve the efficiency of the heart transplantation system~\citep{Papalexopoulos24:Reshaping}. The allocations of other organs, such as lungs~\citep{OPTN25:ContinuousDistribution,Gottlieb17:Lung}, livers~\citep{Kamath01:Model,Allen24:Transplant}, and kidneys~\citep{Abraham07:Clearing,Mayer06:Eurotransplant}, already rely on such computational methods in the US and abroad.

A common data-driven approach to allocating organs relies on predictors of transplant outcomes, such as the expected life-years gained from an operation~\citep{Berrevoets21:Learning,Berrevoets20:OrganITE,Zilberstein26:Near,Zhang24:Development}. The field of \textit{survival analysis} has developed powerful statistical models for estimating such outcomes~\citep{Cox72:Regression,Katzman18:Deepsurv,Lee18:Deephit,Wei92:Accelerated,Nagpal21:Deep}. Yet, when these models are integrated into allocation mechanisms, the aforementioned disconnect surfaces. Survival models are traditionally optimized for and evaluated on metrics such as \textit{Concordance index (C-index)} or average error, which measure aggregate performance across an entire dataset. However, when a donor heart arrives, the matching algorithm does not need perfect point-estimates of survival for all patients; rather, it requires a guarantee that it can identify the single best available match.

\paragraph{Our contributions} As we will demonstrate, matching with a predictor that is optimized for C-index can have arbitrarily bad outcomes. We show that any deterministic algorithm relying on a predictor with near-perfect C-index can obtain a near-zero fraction of the optimal utility (\Cref{prop:deterministic-c-index}). We then prove that no algorithm relying on a predictor with near-perfect C-index can guarantee more utility than random selection (\Cref{prop:randomized-c-index}), showing that C-index is a non-informative measure for allocating even a single donor. This failure is not just restricted to C-index. Most aggregate metrics, such as average error, can also lead to arbitrarily bad outcomes. 

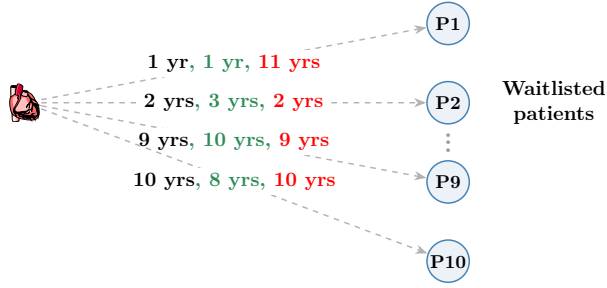
\begin{figure}[t!]
    \centering
    \scalebox{0.7}{\begin{tikzpicture}[
    font=\sffamily,
    >=Stealth,
    patient/.style={circle, draw=patientblue, fill=patientblue!10, thick, inner sep=0pt, minimum size=0.8cm, font=\bfseries\small},
    donor/.style={inner sep=0pt, minimum size=0pt},
    selected_match/.style={->, very thick, edgegreen},
    potential_match/.style={->, dashed, gray!60, thick},
    edge_label/.style={midway, fill=white, inner sep=2.5pt, text=edgegreen, font=\bfseries},
    dots/.style={font=\bfseries\Large, text=gray!80}
]

\node[donor] (d1) at (-5, 2.0)  {\tikz{\pic[scale=0.06]{complexheart};}};

\node[align=center, font=\bfseries] at (5, 2) {Waitlisted\\patients};


\node[patient] (p1) at (3, 3.5) {P1};
\node[patient] (p2) at (3, 2.0) {P2};

\node[dots] at (3, 1.35) {$\vdots$};

\node[patient] (p3) at (3, 0.5) {P9};
\node[patient] (p4) at (3, -1.0) {P10};


\draw[potential_match] (d1) -- (p1) node[edge_label] {\textcolor{black}{1 yr}, \textcolor{edgegreen}{1 yr}, \textcolor{red}{11 yrs}};
\draw[potential_match] (d1) -- (p2) node[edge_label] {\textcolor{black}{2 yrs}, \textcolor{edgegreen}{3 yrs}, \textcolor{red}{2 yrs} }; 
\draw[potential_match] (d1) -- (p3) node[edge_label] {\textcolor{black}{9 yrs}, \textcolor{edgegreen}{10 yrs}, \textcolor{red}{9 yrs} }; 

\draw[potential_match] (d1) -- (p4) node[edge_label] {\textcolor{black}{10 yrs}, \textcolor{edgegreen}{8 yrs}, \textcolor{red}{10 yrs} }; 

\end{tikzpicture}}
    \caption{Illustration of heart transplant allocation with predicted outcomes. The leftmost value shows the unknown, ground-truth patient survival, the middle \textcolor{edgegreen}{green} value shows the predictions with $\ndcgk{1} = 0.9$, and the rightmost \textcolor{red}{red} value shows the predictions with C-index $= 0.8$.}
    \label{fig:heart-matching}
\end{figure}

We take a step towards bridging the gap between predictive modeling in survival analysis and the requirements of downstream allocation policies. While we focus on matching with predicted edge-weights and predictors for survival analysis, our methods provide a template for evaluating and optimizing ML models whose primary purpose is to inform discrete allocation decisions. 

We begin by establishing a formal link between a predictor's $\ndcgk{k}$ and the utility guarantee of downstream allocation (\Cref{thm:random-ndcgk}). We prove that the $\ndcgk{1}$ of a predictor translates to a provable guarantee on the utility of greedy allocation policies (\Cref{cor:greedy-ndcg}), a property not shared by the C-index.~\Cref{fig:heart-matching} illustrates this discrepancy.  

We then introduce the use of \textit{normalized discounted cumulative gain (NDCG)}~\citep{Jarvelin02:Cumulated,Wang13:Theoretical} for survival analysis. NDCG cannot be directly applied to survival analysis due to right censorship: many data points are only represented by a lower bound on their true survival times because the patient is still alive or their follow-up has ended. We propose two novel estimators of NDCG for right censored data, and prove that both provide unbiased estimates of the true \textit{discounted cumulative gain (DCG)}. We show how such estimators can be used to select the model with superior NDCG.

Finally, we propose a method to bootstrap current survival predictors to optimize a model for NDCG. We show using real historical heart transplant data that our estimators of NDCG can accurately identify model superiority and our bootstrapping method significantly improves the NDCG, roughly doubling the $\ndcgk{1}$ from baseline models. These gains are monumental: applying the increase translates to nearly 50,000 additional life years annually in the US alone.\footnote{Assuming 4,000 annual transplants with a median graft survival of 12 years~\citep{Colvin25:OPTN}.}

Our work exposes failures in the current design of high-stakes decision-making systems for organ allocation, and we offer theoretically grounded solutions to these failures. We show it is unsafe to assume that better statistical prediction yields better policy outcomes, unveiling that current mechanisms, including those for allocating lungs, livers, and kidneys, are misaligned with their life-saving objectives and cannot guarantee outcomes better than random selection. Beyond transplantation, for ML to be safely deployed, its predictive components must be aligned to the downstream actions they inform, and our methods support this.

The mismatch between prediction and optimization is studied more broadly in the literature, and our work is the first to connect survival analysis with \textit{decision-focused} and \textit{end-to-end learning}~\citep{Donti17:Task,Wilder19:Melding,Elmachtoub22:Smart,Mandi24:Decision,Capitaine25:Online}. These lines of work focus on aligning ML models with the decision-making tasks they inform. In healthcare, this mismatch has also been recognized in the context of causal treatment effects~\citep{Vanderschueren24:Metalearners,Kamran24:Learning,Frauen25:Treatment,Fernandez22:Causal,Arno26:Rank}. Our paper focuses on survival analysis, which has the unique challenge of right censorship, and predicting the top-ranked candidate. We are the first to employ such techniques for organ allocation. We provide further discussion of related work in~\Cref{sec:related}.
\section{Preliminaries}

We begin by reviewing standard predictive measures from information retrieval and survival analysis. 

\paragraph{DCG and NDCG}

Many information retrieval settings are concerned with providing an accurate ranking of a set of datapoints (\textit{e.g.}, recommendations). Given $N$ inputs, we typically care about the $k$ highest-ranked predictions rather than the entire population. Let $T_i$ denote the relevance (\textit{e.g.}, utility) of the point ranked $i$th by a prediction model where a lower rank means higher utility. In standard settings, the ground-truth relevance is known. The \textit{discounted cumulative gain at $k$ ($\dcgk{k}$)} evaluates the quality of the top $k$ ranked items, $\dcgk{k} = \sum_{i=1}^k \frac{T_i}{\log_2(i + 1)}$. To normalize $\dcg$, we compare it to the \textit{ideal $\dcgk{k}$ ($\idcgk{k}$)}, which is the maximum $\dcgk{k}$ achievable if the ranking were perfectly ordered by the true relevance. The \textit{normalized discounted cumulative gain ($\ndcgk{k}$)} is $\nicefrac{\dcgk{k}}{\idcgk{k}}$. So, $\ndcgk{k} = 1$ represents a perfect ordering of the top $k$ points. For further background on information retrieval, we refer to~\citet{Burges10:Ranknet} and~\citet{Schutze08:Introduction}.

\paragraph{Survival analysis}
In the typical setting for survival analysis, we are given a dataset of individuals $i \in \{1, \dots, N\}$. 
Let $T^*_i$ denote the true unobserved survival time of person $i$ and $C_i$ the censoring time. The observable random variable $T_i = \min\{T^*_i, C_i\}$ and the event indicator $\delta_i = \I\{T^*_i \leq C_i\}$ where $\I$ is the binary indicator function. Let $X_i \in \mathbb{R}^d$ be the baseline covariate vector for patient $i$. Instead of predicting an unbounded survival time, we can also shift the target to predict survival within a fixed horizon, $\tau$, where $T^{\restrict,*}_i = \min\{T^*_i, \tau\}$ and $\delta_i^{\restrict} = \max\left\{\delta_i, \I\{T_i \geq \tau\}\right\}$.  We define $S(t \mid X) = \pr(T^* > t \mid X)$ as the true conditional survival function and $G(t \mid X) = \pr(C > t \mid X)$ as the true conditional censoring survival function. We assume that the covariates $X$ capture both the survival and censoring mechanisms and they are independent conditioned on $X$. 

\begin{assumption}[Conditionally independent censoring]\label{ass:censoring}
$(T^* \perp\!\!\!\perp C) \mid X$.
\end{assumption}


For heart transplantation, we aim to predict how long a new organ will sustain a patient following an operation. A ubiquitous challenge with healthcare datasets is right censorship due to patients stopping reporting. We know the last time a patient reported their condition, not the true event time. 

To adapt $\ndcg$ for survival data, the relevance score becomes the true survival time (or the restricted true survival time). However, for censored patients, the true relevance remains unobservable. Therefore, we require new estimators to compute $\ndcg$ for right censored datasets.

\paragraph{Concordance index}

The standard metric for evaluating predicted rankings in survival analysis is the \textit{concordance index (C-index)}. The C-index measures pairwise accuracy and is defined as the ratio of concordant pairs among all \emph{comparable} pairs of patients. There are different ways of computing the C-index for right censored datasets~\citep{Gonen05:Concordance,Uno11:C}, and we adopt the commonly used Harrell's C-index~\citep{Harrell82:Evaluating}. These variations do not change the underlying principle of the metric. A pair of patients $(i,j)$ is comparable if we definitively know which patient experienced the event of interest (death, graft failure, \textit{etc}.) first. A pair is comparable only if the patient with the shorter observed time experienced the event (\textit{i.e.}, $T_i < T_j$ and $\delta_i =1$). 

\section{Aligning machine learning predictors with policy optimization}

In this section, we analyze how ML predictors interact with the allocation policies they inform. We begin by presenting a motivating example that highlights the limitations of allocating based on a predictor optimized for C-index (or other aggregate metrics). We assume that all utility is non-negative: following a transplant, the time a patient survives cannot be less than zero. 

\subsection{A motivating example}

Consider a fully-observed dataset (\textit{e.g.}, $\delta_i =1 ~~\forall i$), consisting of $N$ transplant candidates and a single donor to be allocated. The goal is to allocate the donor to the patient with the best outcome (\textit{e.g.}, maximize utility). 
Suppose the utility of allocating to patient $i$ is $T_i^* = i$. An allocation algorithm does not know the true utility, but rather relies on a predictive model.

Now suppose a model correctly ranks the utility for patients $i \in[2,N]$, but incorrectly predicts that patient $1$, who has the worst outcome, has the best one. When we evaluate the concordance of this model, we obtain a C-index of $1 - \frac{2}{N}$. As $N$ grows, the C-index quickly approaches $1$ despite the model recommending the patient with the worst outcome.


For $N=100$, the above example results in a C-index of 0.98. This value is much higher than the state-of-the-art models for predicting graft survival following a heart transplant, which are around $0.6$~\citep{Aleksova20:Risk,Lee18:Deephit,Anagnostides25:Policy,Ayers21:Using}. However, it is evident that despite a high C-index, a model can be arbitrarily bad at predicting the top candidate. A greedy algorithm allocating transplants based on this predictor would make catastrophic decisions. The C-index is not the right measure of a predictor that is being leveraged by a decision-making algorithm. We can formalize the failure of matching with concordance for an allocation algorithm that selects only a single candidate using a predictor.

\begin{restatable}{proposition}{deterministiccindex}\label{prop:deterministic-c-index}
    For any deterministic algorithm selecting a single candidate based on a predictor $\hat{f}$, and for any $c \in (0,1)$ and $\rho \in (0,1)$, there exists a set of $N$ candidates and a predictor $\hat{f}$ with C-index at least $c$, such that the algorithm achieves at most a $\rho$ fraction of the optimal utility.
\end{restatable}

Furthermore, not only do deterministic algorithms fail, but \emph{no algorithm}---even a randomized one---relying on concordance can hope to do better than uniform random guessing. 
\begin{restatable}{proposition}{randomizedcindex}\label{prop:randomized-c-index}
    For any algorithm selecting a single candidate based on a predictor $\hat{f}$ and for any $c \in (0,1)$, there exists a set of $N$ candidates and a predictor $\hat{f}$ with C-index at least $c$, such that the algorithm cannot guarantee more expected utility than that of a uniform random selection.
\end{restatable}

This failure is not unique to the C-index. Many aggregate metrics, including standard average errors, also fail to provide utility guarantees. However, if we compute the $\ndcgk{1}$ of the above predictor with $N=100$, we would indeed see the catastrophic performance; the $\ndcgk{1} = 0.01$.


\subsection{Allocating with NDCG}

The failure of C-index arises from its inability to identify high-utility candidates. NDCG, on the other hand, captures exactly this quantity, and we can indeed bound the worst-case utility an allocation algorithm can achieve against the optimal allocation, $U(\opt)$. Specifically, given a predictor $\hat{f}$ with $\ndcgk{k}$ at least $\alpha$, we can randomly select one of the top $k$ ranked candidates proportional to the logarithmic discount in the $\dcgk{k}$ definition. We refer to this algorithm as the \textit{position-weighted allocation algorithm} ($\pwg$) which we present in~\Cref{alg:pwg}.

\begin{restatable}{theorem}{randomndcgk}\label{thm:random-ndcgk}
    Given $N$ candidates and a predictor $\hat{f}$ with $\ndcgk{k}$ at least $\alpha$, algorithm $\pwg$ (\Cref{alg:pwg}) selecting a single candidate based on $\hat{f}$ achieves expected utility at least $\nicefrac{\alpha}{W_k} \cdot U(\opt)$ where $W_k = \sum_{i=1}^k \nicefrac{1}{\log_2(i+1)}$.
\end{restatable}

The bound we obtain is a function of both $\alpha$ and $k$, and degrades approximately linearly in $k$. For $k=5$, $\pwg$ guarantees roughly $\alpha/3$ of the optimal utility. $\ndcgk{k}$ can have higher variance at lower values of $k$, so while optimizing for lower values of $k$ is theoretically better, it may be the case that there are practical tradeoffs for robustness. When evaluating the predictor at $k=1$, the randomized policy reduces into the greedy algorithm, yielding a direct result for greedy allocation.

\begin{restatable}{corollary}{greedyndcg}\label{cor:greedy-ndcg}
    Given $N$ candidates and a predictor $\hat{f}$ with $\ndcgk{1}$ at least $\alpha$, algorithm $\pwg$ selecting a single candidate based on $\hat{f}$ achieves utility at least $\alpha \cdot U(\opt)$.
\end{restatable}

These results stand in stark contrast to the performance of aggregate metrics which fail to guarantee any meaningful utility for single-item allocation. To align a data-driven predictor with the downstream allocation, these results prove that NDCG is a theoretically sound metric to optimize for. 

While \Cref{thm:random-ndcgk} establishes a bound for a single allocation, real-world scenarios often require sequential matching. Over multiple donor arrivals, we can apply the techniques we develop in the remainder of this paper to individual predictors for different representative donor types, consistently optimizing for the $\ndcgk{k}$.

\section{NDCG for censored datasets}

We have shown that optimizing for NDCG is a better target for predictors used by a greedy allocation algorithm than aggregate metrics like C-index. However, since the true relevance score $T^*_i$ is unobservable for right censored data points, we cannot directly compute the standard DCG metrics in the survival analysis context. In this section, we propose two estimators of DCG that account for censoring. We show that both are unbiased whether we use the true survival time or the restricted survival time as the relevance score. It follows from the linearity of expectation that the DCG is unbiased if the relevance is conditionally unbiased given $X_i$. We then discuss how unbiased estimates of DCG translate to estimates of NDCG and evaluating survival models for allocation.

\subsection{Unbiased estimates of relevance}

The first method replaces the unobserved survival time for censored patients with its conditional expectation given a survival function $\hat{S}$. We define the \textit{expected years (EY)} relevance estimator as

\begin{equation*}
    \rel_i^\ey = \delta_i T_i + (1 - \delta_i) \E_{\hat{S}}[T^*_i \mid T^*_i > T_i, X_i].
\end{equation*}

To prove that the EY estimator is unbiased, we need to assume that $\hat{S}$ is unbiased. 

\begin{assumption}[Unbiasedness of $\hat{S}$]\label{ass:survival-specification}
The conditional survival function $\hat{S}$ is conditionally unbiased such that the expected value of the estimated survival times match the true survival time. That is, 
$\E \left[ \E_{\hat{S}}[T^*_i \mid T^*_i > T_i, X_i] \right] = \E[T^*_i \mid T^*_i > T_i, X_i]$.
\end{assumption}

\begin{restatable}[Unbiasedness of EY estimator]{property}{unbiasednessEY}\label{prop:ey}
Under Assumptions \ref{ass:censoring} and \ref{ass:survival-specification}, $\rel^\ey_i$ is a conditionally unbiased estimator of $\E[T^*_i \mid X_i]$.
\end{restatable}

Using the same argument, we can also show that the estimator is unbiased in the restricted setting (\Cref{prop:restricted-ey}). The theoretical unbiasedness of the EY estimator relies on the assumption that the conditional survival estimator is unbiased over the censored population. While relying on nuisance estimators\footnote{A nuisance parameter is one that is not of primary interest but must be accounted for to analyze the target parameters.} is common in statistics, in some ways this approach seems circular: we require a survival model to compute the DCG of another survival model. But this also highlights where the power of our bootstrapping framework stems from. If we can use a survival model as a nuisance estimator, we can also leverage it to bootstrap another model that is optimized for extremal ranking.

An alternative to imputation is the \textit{inverse probability of censoring weighting (IPCW)}~\citep{Graf99:Assessment,Gerds06:Consistent}. IPCW discards censored patients from the evaluation and re-weights the observed instances to account for the discarded population, using the censoring function $\hat{G}(t \mid X)$. The censoring survival function $\hat{G}(t \mid X)$ is the probability of remaining uncensored up to time $t$ given $X$. The IPCW estimator is
\begin{equation*}
    \rel^{\ipcw}_i = \frac{\delta_i T_i}{\hat{G}(T_i \mid X_i)}.
\end{equation*}

In order to prove unbiasedness of the IPCW estimator, we need to assume that $\hat{G}$ is specified such that the inverse weights correctly recover the population expectation. 

\begin{assumption}[Correct specification of censoring model]\label{ass:censoring-calibration}
The conditional censoring model $\hat{G}$ is correctly specified such that it matches the true conditional censoring distribution $G(t \mid X) = \pr(C_i > t \mid X_i)$. We also assume that $\hat{G}(t \mid X) > 0$ for all $t, X$.
\end{assumption}

\begin{restatable}[Unbiasedness of IPCW estimator]{property}{unbiasednessIPCW}\label{prop:ipcw}
Under Assumptions \ref{ass:censoring} and \ref{ass:censoring-calibration}, $\rel^{\ipcw}_i$ is a conditionally unbiased estimator of $\E[T_i^* \mid X_i]$.
\end{restatable}

As is the case for the EY estimator, the IPCW estimator is also unbiased in the restricted setting (\Cref{prop:restricted-ipcw}). Restricting the horizon $\tau$ bounds the maximum weight at $1/\hat{G}(\tau \mid X)$ and is a practical solution for stability since IPCW weights can have high variance. \Cref{sec:appendix-estimators} discusses the estimators and the restricted setting further.

\subsection{NDCG from unbiased relevance}\label{sec:ndcg-bias}

Given unbiased estimates of relevance, linearity of expectation ensures that $\widehat{\dcg}$ and $\widehat{\dcgk{k}}$ are unbiased. However, the IDCG requires sorting these relevances. Because the maximum operator is convex, sorting noisy estimates amplifies positive errors, leading to a positive bias.
As a result, the estimate of NDCG is generally negatively biased. However, the bias does not affect the relative comparisons of two survival models. For a given dataset, $\widehat{\idcg}$ acts as a normalizer determined by the ground-truth scores and any nuisance estimators. It is independent of the models we evaluate. Therefore, given two survival models, $A$ and $B$, if $\widehat{\ndcg}_A > \widehat{\ndcg}_B$, then $\widehat{\dcg}_A > \widehat{\dcg}_B$. The evaluation preserves the relative ranking of the models, allowing us to determine whether model $A$ is superior to model $B$. In addition, $\widehat{\ndcg}$ is scale-consistent with respect to multiplication. If $\widehat{\ndcg}_A = 2\widehat{\ndcg}_B$, it implies model $A$ achieves twice the discounted gain as model $B$ on the dataset since $\widehat{\idcg}$ cancels out in the ratio. 
\section{Optimizing for NDCG via bootstrapping}
\label{sec:bootstrap-model}

In this section, we briefly describe some of the most commonly deployed models for survival prediction, and then present our approach for bootstrapping a survival predictor for superior NDCG. 

\paragraph{Baseline models} We include in our evaluation a suite of common predictors used for survival analysis. These include the non-parametric \textit{Kaplan-Meier (KM)} estimator~\citep{Kaplan58:Nonparametric}, the semi-parametric \textit{Cox regression (Cox)} model~\citep{Cox72:Regression}, the fully parametric \textit{accelerated failure time (AFT)} model~\citep{Wei92:Accelerated}, and the deep neural network models \textit{DeepSurv}~\citep{Katzman18:Deepsurv} and \textit{DeepHit}~\citep{Lee18:Deephit}. More details of these predictors can be found in~\Cref{sec:appendix-baseline-models}. These models are typically evaluated on their C-index. 

We now show how to leverage a conditional survival predictor $\hat{S}(t \mid X)$ to bootstrap another conditional survival predictor that is optimized specifically for NDCG.

\paragraph{Architecture}

Our bootstrapping framework operates using a two-stage approach. First, a base-survival model (\textit{e.g.}, a model from above) is trained on the censored survival data to produce a conditional survival function $\hat{S}(t \mid X)$. Using this baseline, we compute the imputed label given the covariates $X_i$, and a restriction to the horizon, $\tau$, $\hat{y}(X_i, \tau) = T_i + \int_{T_i}^\tau \nicefrac{\hat{S}(t \mid X_i)}{\hat{S}(T_i \mid X_i)} dt$, or if necessary, we default to the restricted mean $\hat{y}(X_i, \tau) = \int_{0}^\tau \hat{S}(t \mid X_i) dt$.
To construct the training labels for the second stage, we create a pseudo-label $y_i^*$ that blends the observed outcomes with the model's conditional expectations to handle censoring. For a patient with observed time $T_i$ and event indicator $\delta_i$, we set $y_i^*$ as $\tau$ if $T_i \ge \tau$, $T_i$ if $T_i < \tau$ and $\delta_i = 1$, and $\hat{y}(X_i, \tau)$ if  $T_i < \tau$ and $\delta_i = 0$. Finally, we train a model, denoted $f_{\hat{S}}(X_i)$,  to predict $y_i^*$ using an objective function designed to optimize the ranking quality and the prediction error. We utilize a gradient-boosted decision tree as the underlying architecture. Other architectures are possible within this framework as well, such as a deep neural network. 

\paragraph{Loss function}

We train the second model using a hybrid objective function to balance the identification of the top patients with the accuracy of the expected survival times. The total loss $\Loss_{\text{hybrid}}$ is a convex combination of the \textit{mean squared error (MSE)} and a pairwise ranking penalty, motivated by LambdaRank~\citep{Burges10:Ranknet}. We combine the two loss functions using a  hyperparameter $\alpha \in [0, 1]$, $\Loss_{\text{hybrid}} = \alpha \Loss_{\text{MSE}} + (1 - \alpha) \Loss_{\text{rank}}$ and evaluate over a resampled mini-batch (\textit{i.e.}, query group) of patients $B$. The regression component, $\Loss_{\text{MSE}} = \frac{1}{|B|}\sum_i (f_{\hat{S}}(X_i) - y_i^*)^2$, provides stability by ensuring the predicted survival time does not deviate too much from the imputed labels.
The ranking component, $\Loss_{\text{rank}}$, optimizes for NDCG through pairwise losses scaled by the change in NDCG. We consider all pairs $(i, j)$ where patient $i$ outlived patient $j$. That is, $y_i^* > y_j^*$. If the predicted difference $f_{\hat{S}}(X_i) - f_{\hat{S}}(X_j)$ is less than a margin $m$, we apply a penalty that is scaled by $|\Delta \text{NDCG}_{i,j}|$, representing the absolute change in the batch's overall NDCG if patient $i$ and $j$ were to swap ranks, $\Loss_{\text{rank}} = \sum_{i, j \in B \mid y_i^* > y_j^*} \frac{1}{2} \max\left\{0, m - \left(f_{\hat{S}}(X_i) - f_{\hat{S}}(X_j)\right)\right\}^2 \cdot |\Delta \ndcg_{i,j}|$. Scaling by the batch-wide $|\Delta \ndcg_{i,j}|$ ensures the model prioritizes a correct ordering of the top-ranked candidates, effectively maximizing the utility of the resulting allocation policy. We leverage the batch-wide metric rather than $\Delta \ndcgk{k}_{i,j}$ to provide a smoother gradient signal.

\section{Experiments}

We utilize the United Network for Organ Sharing (UNOS) patient registry containing clinical data for adult heart transplants in the US dating back to 1987. We provide a summary of the dataset's characteristics in~\Cref{tab:cohort_summary}. Our learning objective is to estimate \textit{graft survival}, defined as the time elapsed from transplant to organ failure or recipient death. To ensure stability, we restrict to predicting up to $\tau=20$ years, aligning with the 95th percentile of censoring times in the data (\Cref{tab:followup_distribution}). 

An alternative goal to graft survival prediction would be to predict \textit{life-years gained (LYG)} of a transplant, which is the difference in years of life for a patient with and without a transplant. While LYG is often  used for allocation, it relies on the graft survival prediction. For heart allocation, graft survival prediction is also the driving factor of LYG: since conditioned on being waitlisted, a patient's survival without a transplant is typically short, the post-transplant outcome dominates LYG~\citep{Colvin25:OPTN}. 


\subsection{Warm-up: Artificial censoring}

We begin by evaluating our bootstrapping framework and our NDCG estimators under \textit{artificial censoring}. We restrict the initial cohort to only patients with observed events, obtaining a ground-truth dataset where the exact survival time $T_i^*$ is known for every individual. We then introduce artificial censoring to simulate the information loss present in real-world clinical registries. We provide further details on the experimental setup in~\Cref{sec:appendix-experimental-setup}.

\paragraph{Results of bootstrapping}

\Cref{tab:main_results_artificial_censoring} presents the comparative performance between baseline survival predictors and their bootstrapped counterparts, and we show the percent gain in NDCG in~\Cref{fig:ndcg-artificial-increase}. We report ground-truth metrics calculated using the hidden true event times. Across all baseline models, our bootstrapping framework consistently yields substantial improvements in $\ndcgk{k}$.  We see that the standard survival models achieve poor $\ndcgk{k}$, with the values consistently below 0.3 for $k = 1$ regardless of the model. In contrast, our approach improves upon this to over $0.4$ and as high as $0.5$ for $\ndcgk{1}$. The increase for $\ndcgk{1}$ is substantial, with the largest improvement over 0.2 when we bootstrap, and our method consistently yielding at least a $50\%$ increase. 


The gain in NDCG is achieved without compromising other metrics; our framework maintains nearly identical (in fact, marginally superior) C-index and AUC scores compared to the standard models. We observe that the full NDCG is around 0.9 across all models.  This is largely due to the density of survival times in the cohort, where the cumulative sum of relevance is dominated by the mass of long-term survivors, making the metric less sensitive to individual swaps. 

\begin{table}[h]
\centering
\resizebox{\textwidth}{!}{
\begin{tabular}{llcccccccc}
\toprule
\textbf{Predictor} & \textbf{Framework} & \textbf{C-Index} & \textbf{AUC} & \textbf{NDCG@1} & \textbf{NDCG@5} & \textbf{NDCG@10} & \textbf{NDCG@50} & \textbf{NDCG@100} & \textbf{NDCG} \\
\midrule
KM & Standard & .500 $\pm .000$ & .500 $\pm .000$ & .230 $\pm .185$ & .213 $\pm .082$ & .229 $\pm .064$ & .253 $\pm .032$ & .270 $\pm .027$ & .850 $\pm .003$ \\
 & + Ours & \textbf{.636 $\pm .003^{**}$} & \textbf{.754 $\pm .011^{**}$} & \textbf{.437 $\pm .261^{**}$} & \textbf{.413 $\pm .123^{**}$} & \textbf{.424 $\pm .092^{**}$} & \textbf{.427 $\pm .040^{**}$} & \textbf{.446 $\pm .027^{**}$} & \textbf{.902 $\pm .003^{**}$} \\
\hline
Cox & Standard & .623 $\pm .004$ & .763 $\pm .011$ & .288 $\pm .211$ & .343 $\pm .114$ & .360 $\pm .082$ & .411 $\pm .045$ & .424 $\pm .032$ & .898 $\pm .003$ \\
 & + Ours & \textbf{.637 $\pm .003^{**}$} & \textbf{.770 $\pm .010^{**}$} & \textbf{.494 $\pm .256^{**}$} & \textbf{.425 $\pm .117^{**}$} & \textbf{.425 $\pm .095^{**}$} & \textbf{.426 $\pm .045^{**}$} & \textbf{.442 $\pm .034^{**}$} & \textbf{.903 $\pm .003^{**}$} \\
\hline
AFT & Standard & .631 $\pm .003$ & .744 $\pm .009$ & .264 $\pm .228$ & .246 $\pm .107$ & .272 $\pm .093$ & .352 $\pm .045$ & .389 $\pm .032$ & .895 $\pm .003$ \\
 & + Ours & \textbf{.639 $\pm .003^{**}$} & \textbf{.758 $\pm .012^{**}$} & \textbf{.401 $\pm .246^{**}$} & \textbf{.420 $\pm .114^{**}$} & \textbf{.421 $\pm .084^{**}$} & \textbf{.421 $\pm .041^{**}$} & \textbf{.443 $\pm .025^{**}$} & \textbf{.903 $\pm .003^{**}$} \\
\hline
DeepSurv & Standard & .626 $\pm .003$ & .766 $\pm .011$ & .272 $\pm .234$ & .276 $\pm .144$ & .330 $\pm .101$ & .375 $\pm .042$ & .405 $\pm .030$ & .897 $\pm .003$ \\
 & + Ours & \textbf{.637 $\pm .003^{**}$} & \textbf{.769 $\pm .011^{**}$} & \textbf{.406 $\pm .254^{**}$} & \textbf{.400 $\pm .154^{**}$} & \textbf{.403 $\pm .114^{**}$} & \textbf{.423 $\pm .052^{**}$} & \textbf{.439 $\pm .037^{**}$} & \textbf{.903 $\pm .003^{**}$} \\
\hline
DeepHit & Standard & .626 $\pm .003$ & .748 $\pm .012$ & .216 $\pm .220$ & .283 $\pm .128$ & .318 $\pm .107$ & .371 $\pm .055$ & .395 $\pm .034$ & .895 $\pm .004$ \\
 & + Ours & \textbf{.637 $\pm .003^{**}$} & \textbf{.765 $\pm .012^{**}$} & \textbf{.450 $\pm .272^{**}$} & \textbf{.376 $\pm .123^{**}$} & \textbf{.387 $\pm .090^{**}$} & \textbf{.422 $\pm .049^{**}$} & \textbf{.435 $\pm .032^{**}$} & \textbf{.902 $\pm .003^{**}$} \\
\bottomrule
\multicolumn{10}{l}{{\small $^{*} p < .05, ^{**} p < .01$ (Wilcoxon signed-rank test vs. standard model)}} \\
\end{tabular}
}
\caption{Performance of models on ground-truth metrics under artificial censoring.}
\label{tab:main_results_artificial_censoring}
\end{table}

\begin{figure}[h]
    \centering
    \includegraphics[width=0.5\linewidth]{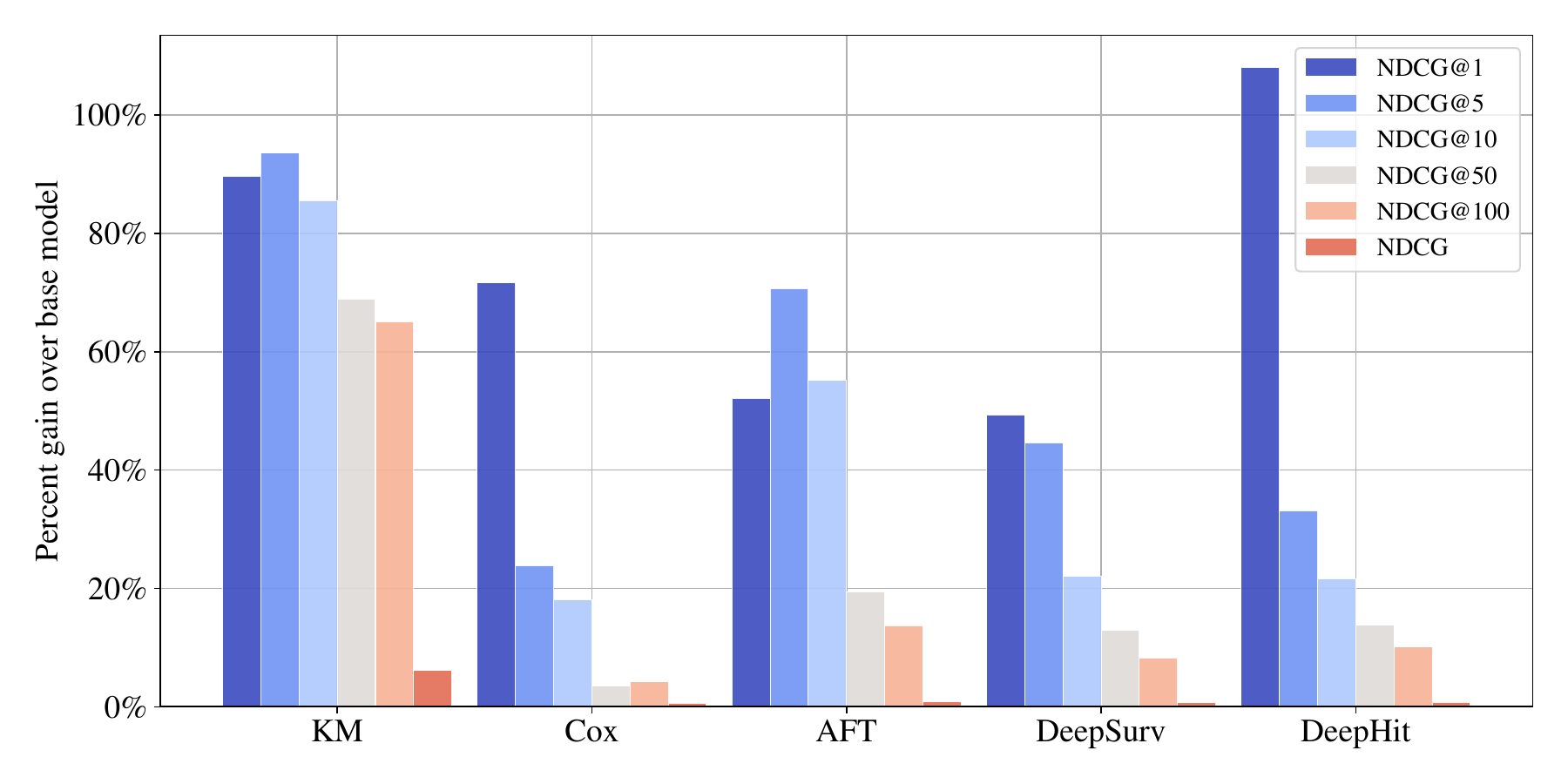}
    \caption{Average $\%$ gain in $\ndcgk{k}$ of bootstrapped model over baseline predictor.}
    \label{fig:ndcg-artificial-increase}
\end{figure}

\paragraph{Accuracy of NDCG estimators}

We evaluate the fidelity of the EY and IPCW estimators for NDCG by comparing their outputs against the ground-truth NDCG available from our artificial censoring. For the EY estimator, we utilize all baseline survival predictors as nuisance models. For the IPCW estimator, we employ the KM and Cox nuisance models. \Cref{fig:mae-heatmap} displays the \textit{mean absolute error (MAE)} for each estimator across different values of $k$. We observe that the estimation error generally decreases as $k$ increases, approaching near-zero error for the full NDCG estimation. This trend is expected, as larger values of $k$ aggregate more points, which averages out individual estimation errors. For $\ndcgk{1}$, the MAE exceeds 0.1. The IPCW estimator has the lowest average error as $k$ increases. In general, we expect some error, particularly due to the bias detailed in~\Cref{sec:ndcg-bias}.

Despite the moderate absolute error, the estimators demonstrate strong correlation when compared to the ground-truth NDCG (\Cref{fig:ndcg-scatter}). The EY estimator, aggregated across all nuisance models, exhibits a high degree of correlation with the true NDCG, achieving a Spearman rank correlation coefficient of $0.8$ for $k\in \{1,5,10, 50\}$.
In contrast, the IPCW estimator displays higher variance; while individual estimates are noisier, the mean of the distribution tracks the ground-truth effectively. We observe some positive bias in the EY estimator likely due to it over-estimating population survival. 

Finally, we evaluate the \textit{relative} accuracy of our estimators. Specifically, their reliability in performing model selection. We report the pairwise concordance: for all pairs of survival models, $(A, B)$, we determine the frequency with which the estimator correctly identifies the superior model. That is, whether $\widehat{\ndcg}_A > \widehat{\ndcg}_B$ given that $\ndcg_A > \ndcg_B$. As shown in~\Cref{tab:estimators_artificial_censoring}, our estimators are consistent at predicting the relative ranking of models. We also evaluate an ensemble estimation, which averages the estimates over all nuisance models. The EY estimator demonstrates superior reliability, correctly identifying the better-performing model in over 80\% of cases, and reaching 90\% accuracy in many instances. While the IPCW estimator is more volatile, it still achieves over 75\% pairwise accuracy for $k=1$. These results establish that both estimators, particularly the EY approach, are robust tools for model selection in the presence of censoring.

\subsection{Evaluating on the full dataset}

Having validated our estimators and our bootstrapping approach under artificial censoring, we now apply our methods to the complete UNOS registry. Unlike the artificial setup, the survival times of censored patients in this dataset are truly unknown. We rely on the NDCG estimators to assess the ranking performance. We summarize the results in~\Cref{tab:full_results} and visually show the percent gain in the NDCG estimates in~\Cref{fig:ndcg-gain}. The results show a consistent performance gain across all baseline predictors when integrated into our bootstrapping framework. The C-index and \textit{AUC at 5 years (AUC@5)} again are very slightly improved for every baseline model tested when using bootstrapping. The AUC@5 measures the model's ability to distinguish between graft failure and survival at five years post-transplant. These results suggest that the architecture we use for the bootstrapped model may be an effective model for this predictive task.

When evaluating the estimated $\ndcgk{k}$, we see the true power of bootstrapping. Across both EY and IPCW estimators, our framework outperforms standard survival models, and by substantial margins at small values of $k$, regardless of the underlying imputer. While the $\ndcgk{1}$ is comparable to the $\ndcgk{5}$ for our approach, the variance of the estimation decreases as $k$ increases. This is one reason why we may examine a larger $k$ when selecting a predictor. The IPCW estimates exhibit higher variance, but they largely trend in the same direction as the EY estimates, confirming that our approach is effective at improving the identification of high-utility candidates. 


\begin{table}[t]
\centering
\resizebox{0.8\textwidth}{!}{
\begin{tabular}{ll cc cc cc}
\toprule
 & & & & \multicolumn{2}{c}{\textbf{EY NDCG@$k$}} & \multicolumn{2}{c}{\textbf{IPCW NDCG@$k$}} \\
\cmidrule(lr){5-6} \cmidrule(lr){7-8}
\textbf{Predictor} & \textbf{Framework} & \textbf{C-Index} & \textbf{AUC@5} & \textbf{1} & \textbf{5} & \textbf{1} & \textbf{5} \\
\midrule
KM & Standard & .500 $\pm .000$ & .500 $\pm .000$ & .647 $\pm .288$ & .610 $\pm .158$ & .430 $\pm .388$ & .350 $\pm .189$ \\
 & + Ours & \textbf{.618 $\pm .004^{**}$} & \textbf{.643 $\pm .005^{**}$} & \textbf{.888 $\pm .184^{**}$} & \textbf{.833 $\pm .108^{**}$} & \textbf{.758 $\pm .523^{**}$} & \textbf{.688 $\pm .284^{**}$} \\
\hline
Cox & Standard & .601 $\pm .004$ & .627 $\pm .005$ & .699 $\pm .271$ & .729 $\pm .100$ & .155 $\pm .414$ & .198 $\pm .202$ \\
 & + Ours & \textbf{.618 $\pm .004^{**}$} & \textbf{.652 $\pm .005^{**}$} & \textbf{.834 $\pm .066^{**}$} & \textbf{.800 $\pm .056^{**}$} & \textbf{.370 $\pm .468^{**}$} & \textbf{.427 $\pm .260^{**}$} \\
\hline
AFT & Standard & .599 $\pm .004$ & .636 $\pm .004$ & .600 $\pm .312$ & .682 $\pm .153$ & .068 $\pm .129$ & .081 $\pm .068$ \\
 & + Ours & \textbf{.617 $\pm .004^{**}$} & \textbf{.652 $\pm .005^{**}$} & \textbf{.810 $\pm .194^{**}$} & \textbf{.778 $\pm .105^{**}$} & \textbf{.128 $\pm .179$} & \textbf{.178 $\pm .130^{**}$} \\
\hline
DeepSurv & Standard & .601 $\pm .004$ & .627 $\pm .005$ & .600 $\pm .333$ & .622 $\pm .142$ & .183 $\pm .302$ & .112 $\pm .113$ \\
 & + Ours & \textbf{.617 $\pm .004^{**}$} & \textbf{.650 $\pm .005^{**}$} & \textbf{.759 $\pm .169$} & \textbf{.756 $\pm .094^{**}$} & \textbf{.269 $\pm .521$} & \textbf{.355 $\pm .334^{**}$} \\
\hline
DeepHit & Standard & .590 $\pm .004$ & .582 $\pm .008$ & .649 $\pm .284$ & .650 $\pm .130$ & .192 $\pm .458$ & .175 $\pm .178$ \\
 & + Ours & \textbf{.613 $\pm .003^{**}$} & \textbf{.620 $\pm .006^{**}$} & \textbf{.730 $\pm .306^{*}$} & \textbf{.769 $\pm .102^{**}$} & \textbf{.601 $\pm .659^{**}$} & \textbf{.487 $\pm .257^{**}$} \\
\bottomrule
\multicolumn{8}{l}{{\small $^{*} p < .05, ^{**} p < .01$ (Wilcoxon signed-rank test vs standard model)}} \\
\end{tabular}
}
\caption{Performance of models on metrics and estimated NDCG (avg  over all nuisance models).}
\label{tab:full_results}
\end{table}

\begin{figure*}[t]
    \centering
    \begin{subfigure}[b]{0.45\textwidth}
        \centering
        \includegraphics[width=\linewidth]{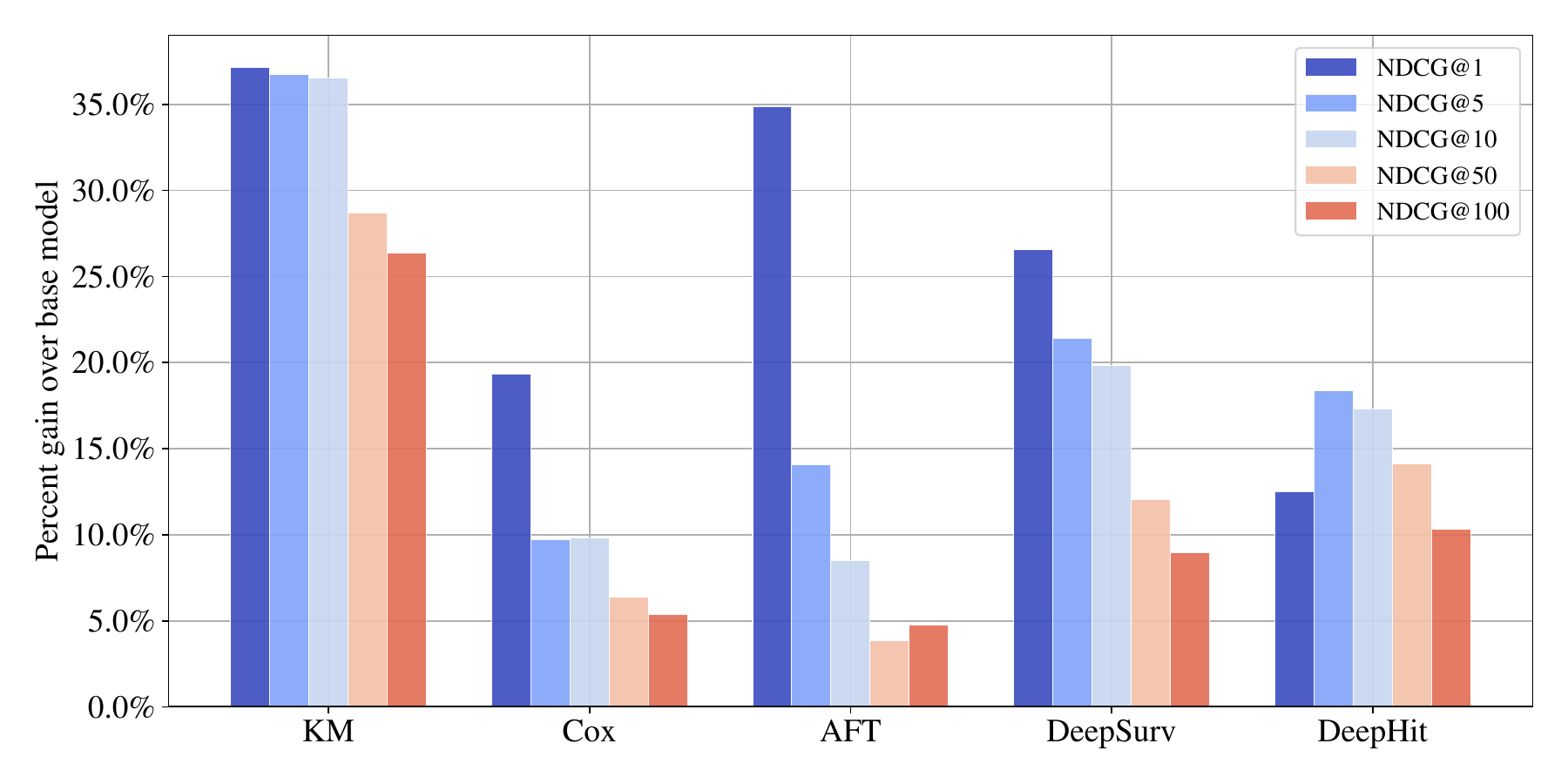}
        \caption{Average $\%$ gain of EY estimation.}
        \label{fig:ndcg-gain-ey}
    \end{subfigure}
    \hfill
    \begin{subfigure}[b]{0.45\textwidth}
        \centering
        \includegraphics[width=\linewidth]{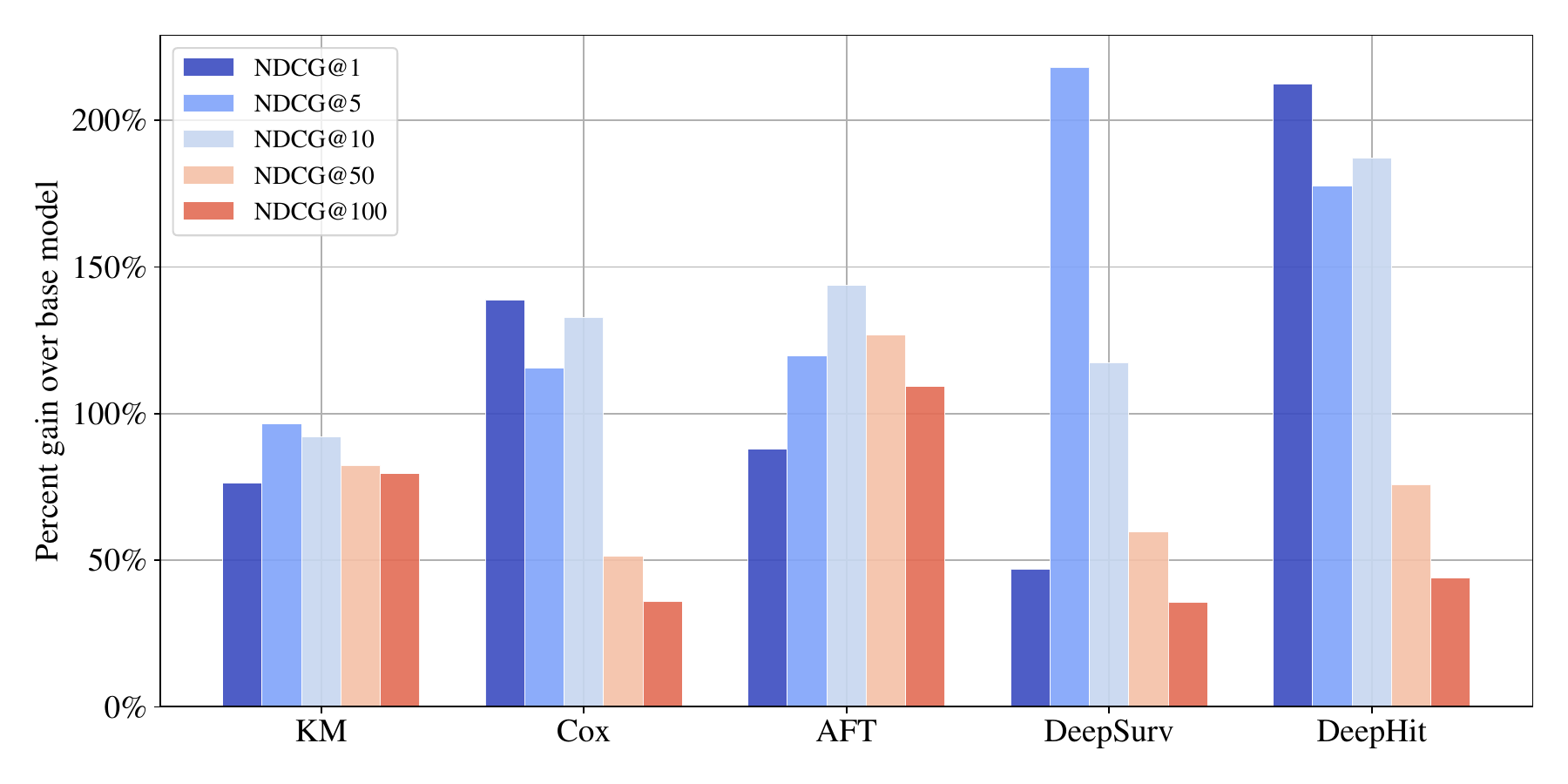}
        \caption{Average $\%$ gain of IPCW estimation.}
        \label{fig:ndcg-gain-ipcw}
    \end{subfigure}
    
    \caption{Average $\%$ gain in $\ndcgk{k}$ estimations of bootstrapped model over baseline predictor.}
    \label{fig:ndcg-gain}
\end{figure*}
\section{Limitations}
\label{sec:limitations}

Relying on real-world historical data introduces inherent limitations. The UNOS registry, like many medical databases, contains erroneous entries and missing data. We preprocess the dataset to impute missing entries and filter inconsistent values.  However, there may be  unobserved confounders absent from the registry which influence outcomes. Systemic biases also present in historical clinical decisions can inadvertently be translated to the predictors. Despite this limitation, our bootstrapping approach is complementary to future advancements in the underlying survival predictors.

The unbiasedness of our DCG estimators relies on strong assumptions regarding the nuisance models. While theoretically necessary, the assumptions are not always achievable in practice.  If the underlying survival model suffers from miscalibration, the EY estimator will inherit this bias. The IPCW estimator similarly requires that the censoring distribution is correctly specified. If censoring is highly informative or the probability of being uncensored approaches zero, the IPCW weights lead to very high-variance estimates. This motivates the use of restricted horizons. However, restricted horizons come at the cost of expressivity. For our application, this restriction is not a major limitation (since the median graft survival is well under 20 years), but it could be a barrier to other applications. Finally, even with unbiased DCG estimators, the resulting NDCG estimate is not unbiased. Our framework reliably determines relative performance and aligns the learning with the allocation, but cannot determine the exact ground-truth NDCG. These limitations underscore the difficulty of predictive tasks under right censorship and promote future work in bias mitigation and model calibration. 
\section{Conclusions}
\label{sec:conclusions}

We addressed the misalignment between predictive components for survival analysis and the requirements of decision making.  We demonstrated that predictors with standard aggregate metrics cannot guarantee any utility when used for allocation. To bridge this gap, we established $\ndcgk{1}$ as a theoretically grounded measure that directly translates to allocation performance.  

In addition, we developed novel estimators of NDCG for right censored datasets and proposed a bootstrapping approach to optimize survival models for NDCG. Our empirical results on real heart transplant data showed the effectiveness of our estimators at determining model superiority and the substantial increase in NDCG when using bootstrapping. By aligning the objectives of survival modeling with allocation optimization, our framework provides a scalable template for improving outcomes in organ transplantation and other domains of decision-making under uncertainty.
\section*{Acknowledgments}

Tuomas Sandholm and his PhD students Ioannis Anagnostides and Itai Zilberstein are supported by NIH award A240108S001, the Vannevar Bush Faculty Fellowship ONR N00014-23-1-2876, and National Science Foundation grant RI-2312342. Itai Zilberstein is also supported by the NSF Graduate Research Fellowship Program under grant DGE2140739. Any opinions, findings, and conclusions or recommendations expressed in this material are those of the author(s) and do not necessarily reflect the views of the funding agencies.

\bibliography{references}

\appendix
\section{Related work}
\label{sec:related}

\paragraph{Decision-focused learning and end-to-end learning}

A closely related line of work is \emph{decision-focused learning (DFL)}. As in our paper, the DFL framework is motivated by the fact that, in machine learning, optimization is typically based on estimators. These two pieces are often treated in isolation by typical approaches~\citep{Wilder19:Melding}. Specifically, a predictive model is first trained through some measure of accuracy, for example, mean squared error. Then that model's predictions are given as input to an optimization algorithm in order to make a decision. While this two-stage approach is justified when the predictive model is highly accurate, in complex tasks state-of-the-art models will inevitably be imperfect. The training process involves tradeoffs as to the nature of such errors. When prediction and optimization are treated in isolation, it creates a critical misalignment.

To address this misalignment, \citet{Donti17:Task} proposed an end-to-end approach for learning ML models in a way that directly captures the final task-based objectives for which the models will be used; this approach is referred to as \emph{end-to-end model learning}. Similarly, decision-focused learning endeavors to align decision and learning~\citep{Wilder19:Melding,Elmachtoub22:Smart,Mandi24:Decision,Capitaine25:Online}. This was pioneered by~\citet{Wilder19:Melding}, who studied certain classes of combinatorial optimization problems, and is also known as integrated decision learning~\citep{Tao26:Necessary}. Closer to our work is the recent paper of~\citet{Capitaine25:Online}, which examines DFL in dynamic environments where the objective and data distribution evolve over time. For a survey on this rapidly growing line of work, we refer to~\citet{Mandi24:Decision}. 

There are two key challenges in applying those prior frameworks in our setting: i) the presence of right censorship, which significantly complicates statistical estimation, and ii) the objective we want to optimize, namely $\ndcgk{1}$, is highly sensitive to distribution shifts.

\paragraph{Algorithms with predictions} 

There is a flourishing line of work on algorithm design through the use of \emph{predictions}~\citep{Mitzenmacher22:Algorithms}, where the goal is to improve performance under reliable estimators and revert back to worst-case performance when the predictors are inaccurate. Heart transplant allocation can also be viewed from that perspective. As motivated above, the interplay between prediction quality and algorithm design is at the core of our approach, driven by the observation that the estimation part needs to be informed by the policy optimization component. These two are typically treated in isolation in prior papers in the line of work on algorithms with predictions.

\paragraph{Calibration} Prediction that aligns with downstream decision tasks can also be accomplished through \emph{calibration}~\citep{Dawid82:Well, Wang23:Calibration,Foster98:Asymptotic}, as it effectively allows treating the estimated quantities as probabilities; exploring how calibration can be used in cases where there is right censorship is an interesting direction for future work. To connect with the line of work on algorithms with predictions, recent work by~\citet{Shen25:Algorithms} studied the ski rental problem under calibrated predictors.

Moreover, the interaction between estimation and downstream decisions is also central in the framework of performative prediction~\citep{Perdomo20:Performative}, where the underlying distribution in the estimation part shifts due to the strategic decisions of the population.

\paragraph{Recommender systems} Another closely related line of research is that of recommender systems~\citep{Burke11:Recommender,Wang13:Theoretical,Rossetti16:Contrasting,Jarvelin02:Cumulated}, which focuses on identifying the optimal items to present to a user. Learning-to-rank~\citep{Cao07:Learning,Burges10:Ranknet} is a foundational technique for training a recommender system, and is the underlying learning objective of our bootstrapping approach. While the connection between learning-to-rank and allocation has been explored in domains such as healthcare resource management~\citep{Kamran24:Learning} and ad auctions~\citep{Richardson07:Predicting}, the formal link between ranking metrics and downstream allocation utility has not been previously established.

\paragraph{Survival prediction and organ allocation}
Survival analysis, a specialized field of time-to-event prediction, focuses on modeling the distribution of event times under conditions of right censorship~\citep{Kaplan58:Nonparametric,Wei92:Accelerated}. Unlike standard regression tasks, the presence of censoring means that for many subjects, we only possess a lower bound on the true event time, rather than an exact observation. Cox regression~\citep{Cox72:Regression} is one of the foundational statistical models for survival analysis. Cox models, which are often evaluated via concordance, are commonly used in organ allocation~\citep{Papalexopoulos24:Reshaping,OPTN25:ContinuousDistribution,Dickerson15:Futurematch}. More recently, deep learning solutions have been developed to extend beyond semi-parametric models for survival analysis~\citep{Lee18:Deephit,Nagpal21:Deep,Katzman18:Deepsurv}.

The application of matching donor organs to patients for transplantation is a high-stakes application of machine learning and matching algorithms~\citep{Su04:Patient,Abraham07:Clearing,Awasthi09:Online,Dickerson15:Futurematch,Dickerson16:Position,Berrevoets20:OrganITE,Berrevoets21:Learning,Anagnostides26:Position,Shojaee21:Adaptively}. Most data-driven approaches rely on \textit{predictions} of transplant outcomes~\citep{Berrevoets20:OrganITE,Berrevoets21:Learning,Dickerson15:Futurematch,Papalexopoulos24:Reshaping,Shojaee21:Adaptively,Zilberstein26:Learning}. These existing methods treat the predictive model as a black-box optimized for aggregate statistical accuracy. While some work has focused on making the predictions more robust~\citep{Berrevoets21:Learning}, the predictive objectives still remain decoupled from the combinatorial requirements of the matching algorithm. Because these predictors are trained in isolation from the downstream matching objective, the resulting allocation mechanisms lack provable bounds on solution quality relative to the true underlying utility.

\paragraph{Conditional average treatment effect estimation} One of the most similar lines of work to ours concerns estimating \textit{conditional average treatment effect (CATE)}, a causal inference problem. Learning-to-rank has also been identified as a core task for CATE when allocating treatments~\citep{Kamran24:Learning,Fernandez22:Causal}. Similar to our bootstrapping approach, re-training a ranker based on an underlying predictive model is a common technique~\citep{Vanderschueren24:Metalearners,Frauen25:Treatment,Arno26:Rank}. Our work differentiates itself from this line of work in multiple ways. First, survival analysis revolves around a different predictive task from the causal treatment effects, which is more akin to predicting the life-years gained. Prior work on CATE has also not addressed right censorship, a distinguishing challenge of survival analysis.  Finally, in general, these prior works focus on ranking an entire set of patients, whereas our allocation is most concerned with predicting the top-$k$ candidates correctly for very small values of $k$ (\textit{i.e.}, $k = 1$).
\setcounter{table}{0}
\renewcommand{\thetable}{A\arabic{table}}

\setcounter{figure}{0}
\renewcommand{\thefigure}{A\arabic{figure}}

\section{Omitted algorithms and proofs}
\label{sec:appendix-proofs}

We begin by proving our claims regarding allocation using a predictor to determine edge weights with only a guarantee on the concordance index of the predictor. 

\deterministiccindex*

\begin{proof}
    Let $N \ge 2$ and consider an arbitrary ranking of the $N$ candidates determined by $\hat{f}$. Let $\pi(i) \in \{1, \dots, N\}$ denote the predicted rank of candidate $i$, where a lower $\pi(i)$ corresponds to a higher utility.

    Given this predicted ranking, the deterministic policy will select a specific candidate. We will adversarially construct the true utilities based on the policy's selection. 

    For a small constant $\epsilon > 0$, we assign the true utilities, $T_i$, as follows.
    \begin{itemize}
        \item For the selected candidate $i$, $T_i = \rho - \epsilon \cdot \pi(i)$
        \item For the $N-1$ non-selected candidates $j$, $T_j = 1 + \epsilon \cdot (N - \pi(j))$
    \end{itemize}
    
    The true utility of the selected candidate is less than the utility of any candidate not selected. The policy achieves a total utility of
    $$ U(\alg) = \rho - \epsilon \cdot \pi(i) < \rho. $$
    Because $N \ge 2$, the optimal policy will select a candidate not selected by the algorithm, achieving a utility of
    $$ U(\opt) = 1 + \epsilon \cdot (N - \pi(j)) > 1. $$
    The ratio of the optimal utility achieved by the policy is therefore upper bounded by $\rho$.

    We now calculate the C-index of the predictor $\hat{f}$. The total number of comparable pairs is $\binom{N}{2}$. All $\binom{N-1}{2}$ pairs excluding candidate $i$ are perfectly ranked. The incorrectly ranked pairs occur when comparing the selected candidate to the rest. There are at most $(N-1)$ such pairs. The C-index is bounded below by
    \[
        1 - \frac{(N-1)}{\binom{N}{2}} = 1 - \frac{2}{N}.
    \]
    As $N \to \infty$, the C-index approaches $1$. We can choose a sufficiently large $N > 2/(1-c)$.  For such an $N$, the predictor achieves C-index at least $c$ while the allocation policy is restricted to a $\rho$ fraction of the optimal utility.
\end{proof}

We continue with the proof of~\Cref{prop:randomized-c-index}.

\randomizedcindex*

\begin{proof}
    Consider any algorithm that assigns a probability $p_i \in [0,1]$ to each candidate $i \in \{1,\dots ,N\}$ based on the utility determined by $\hat{f}$. We assume $\sum_{i=1}^N p_i = 1$. Let $\Delta$ be the total probability mass on the candidate with the highest probability. Let $\pi(i) \in \{1, \dots, N\}$ denote the predicted rank of candidate $i$, where a lower $\pi(i)$ corresponds to a higher utility.

    We assign to the candidate with the highest probability a true, unknown utility of $\rho - \epsilon \cdot \pi(i)$ and a utility of $1 + \epsilon \cdot (N - \pi(i))$ to the remaining $N-1$ candidates for any $\rho \in (0,1)$ and $\epsilon > 0$.  

    As seen in the proof of~\Cref{prop:deterministic-c-index}, if the predictor correctly ranks all the patients except the patient with the largest probability of selection, it achieves a C-index at least $c$ for sufficiently large $N$ as 
    \[
        1 - \frac{(N-1)}{\binom{N}{2}} = 1 - \frac{2}{N}.
    \]

    The expected utility of the algorithm is
    \[ \E[U(\alg)] = \sum_{i=1}^N p_i T_i \le \Delta \cdot \rho + (1 - \Delta) \cdot (1+ \epsilon N) = 1 - \Delta(1 - \rho) + (1-\Delta) \epsilon N. \]
    The optimal offline utility is $U(\opt) \ge 1$. The expected fraction of the optimal utility is then
    \[ 
    \frac{\E[U(\alg)]}{\E[U(\opt)]} \le \frac{1 - \Delta(1 - \rho) + (1-\Delta) \epsilon N}{1} \le 1 - \Delta(1 - \rho) 
    \]
    as $\epsilon$ tends to $0$ for a fixed $N$.
    
    Since $\rho \in (0,1)$, this expected ratio is maximized when $\Delta$ is minimized. The minimum value of the  largest probability is achieved when the distribution is uniform. Any non-uniform policy necessarily has a larger $\Delta$ and can obtain a strictly lower expected ratio of the optimal utility as $N$ grows.
\end{proof}

We move on to the proof of~\Cref{thm:random-ndcgk}, and provide the pseudocode for~\Cref{alg:pwg}. Recall that we assume non-negative weights $(T_i \ge 0~~\forall_i)$ as in our application patients cannot have passed away in the past. 

\begin{algorithm}[h]
\caption{Position-weighted allocation ($\pwg$)}
\label{alg:pwg}
\begin{algorithmic}[1]
\REQUIRE $N$ candidates, predictor $\hat{f}$ with $\ndcgk{k} \ge \alpha$.
\STATE Order the $N$ candidates according to $\hat{f}$.
\STATE Set $w_i = \nicefrac{1}{\log_2(i+1)}$ and $W_k = \sum_{i=1}^k w_i$.
\STATE Randomly allocate to one of the top $k$ ranked candidates proportional to $\nicefrac{w_i}{W_k}$.
\end{algorithmic}
\end{algorithm}

\randomndcgk*

\begin{proof}

    Let $\pi(i) \in \{1, \dots, N\}$ denote the index of the candidate at rank $i$ ordered by $\hat{f}$. We denote the true utility of candidate $i$ as $T_i$ and assume without loss of generality that $T_i \ge T_j$ for all $i \le j$.  $\pwg$ selects rank $i \in [1,k]$ with probability $w_i/W_k$. The expected utility of $\pwg$ is
    \[ \E[\pwg] = \sum_{i=1}^k \frac{w_i}{W_k} T_{\pi(i)} = \frac{\dcgk{k}}{W_k}. \]

    Since $\ndcgk{k} \ge \alpha$, $\dcgk{k} \ge \alpha \idcgk{k}$, so
    \[ 
    \E[\pwg] \ge \frac{\alpha\idcgk{k}}{W_k}.
    \]
    The worst case occurs when $\idcgk{k}$ is minimized. Since $T_1$ is the utility of the best candidate (and also the utility of $\opt$),

    \[ 
    \idcgk{k} = T_1 + \sum_{i=2}^k T_{i}w_i.
    \]
    This quantity is minimized exactly when $T_{i} = 0$ for all $i \in [2,k]$ since the weights are non-negative. This completes the proof as
    \[ 
    \E[\pwg] \ge \frac{\alpha T_1}{W_k} \ge \frac{\alpha}{W_k} U(\opt).
    \]
    
\end{proof}

The proof of~\Cref{cor:greedy-ndcg} follows immediately since $W_1 = 1$. Next, we prove the unbiasedness of our NDCG estimators. 

\unbiasednessEY*

\begin{proof}
We can write the expected value of $\rel^{\ey}_i$ based on the event indicator $\delta_i$,
\begin{equation*}
    \mathbb{E}[\rel^{\ey}_i \mid X_i ] = \E[\delta_i T_i \mid X_i] + \E\left[ (1 - \delta_i) \E_{\hat{S}}[T^*_i \mid T^*_i > T_i, X_i] \mid X_i \right].
\end{equation*}
We consider the two terms and values of $\delta_i$.
\begin{enumerate}
    \item $\delta_i = 1$: The event is fully observed, so $T_i = T^*_i$ which implies $\E[\delta_i T_i \mid X_i] = \E[\delta_i T^*_i \mid X_i]$. 
    \item $\delta_i = 0$: The observation is censored, so $T^*_i > T_i$. Under Assumption \ref{ass:censoring},  $\delta_i = 0$ does not influence the survival time $T^*_i$ other than the survival up to the censoring time. It follows from Assumption \ref{ass:survival-specification} and the Law of Iterated Expectation, that 
    $$
    \E \left[ (1 - \delta_i) \E_{\hat{S}}[T^*_i \mid T^*_i > T_i, X_i] \mid X_i \right] = \mathbb{E}[(1 - \delta_i) T^*_i \mid X_i].
    $$
\end{enumerate}
Combining both cases, 
$$
\E[\delta_i T^*_i \mid X_i] + \E[(1 - \delta_i) T^*_i \mid X_i] = \E\left[ (\delta_i + 1 - \delta_i) T^*_i \mid X_i \right] = \E[T^*_i \mid X_i].
$$ 
\end{proof}

Continuing with the IPCW estimator. 

\unbiasednessIPCW*

\begin{proof}
Note that $\delta_i = \I\{C_i \geq T^*_i\}$ and $T_i = T^*_i$ when $\delta_i = 1$. Since the numerator evaluates to zero when $\delta_i = 0$, we can substitute $T^*_i$ for $T_i$ inside the expectation without changing the expected value. Under Assumption \ref{ass:censoring} and the Law of Iterated Expectation, we get
$$
\E_{T^*, C, X} \left[ \frac{\I\{C_i \geq T^*_i\} T^*_i}{\hat{G}(T^*_i \mid X_i)} ~\Bigg|~ X_i \right] = \E_{T^*, X} \left[ \E_{C} \left[ \frac{\I\{C_i \geq T^*_i\} T^*_i}{\hat{G}(T^*_i \mid X_i)} ~\Bigg|~ T^*_i, X_i \right] ~\Bigg|~ X_i \right].
$$
Since $T^*_i$ is constant relative to the inner expectation over $C$, we factor it out, giving
$$
   \E_{T^*, X} \left[ T^*_i \cdot \E_C \left[ \frac{\I\{C_i \geq T^*_i\}}{\hat{G}(T^*_i \mid X_i)}  ~\Bigg|~ T^*_i, X_i \right] ~\Bigg|~ X_i \right].
$$
By Assumption \ref{ass:censoring-calibration}, the inner expectation equals 1. Therefore,
$\E_{T^*, X} [T^*_i \cdot 1 \mid X_i] = \E[T^*_i \mid X_i]$.

\end{proof}

\section{Further discussion of NDCG estimators}
\label{sec:appendix-estimators}

The restricted EY estimator is $\rel^{\restrict,\ey}_i = \delta^{\restrict}_i T^{\restrict}_i + \left(1 - \delta^{\restrict}_i \right) \E_{\hat{S}} \left[T^{\restrict,*}_i \mid T^*_i > T_i, X_i \right]$.

If we again assume the conditional survival function is conditionally unbiased, we obtain that the restricted estimator provides conditionally unbiased estimates of the relevance. 

\begin{assumption}[Restricted unbiasedness of $\hat{S}$]\label{ass:survival-specification-restricted}
The conditional survival function $\hat{S}$ is conditionally unbiased such that the expected value of the restricted estimated survival times match the true restricted survival time. That is, 
$\E \left[ \E_{\hat{S}}[T^{\restrict,*}_i \mid T^*_i > T_i, X_i] \right] = \E[T^{\restrict,*} \mid T^*_i > T_i, X_i]$.
\end{assumption}

\begin{property}[Unbiasedness of restricted EY estimator]\label{prop:restricted-ey}
Under Assumptions \ref{ass:censoring} and \ref{ass:survival-specification-restricted}, $\rel^{\restrict,\ey}_i$ is a conditionally unbiased estimator of $\E\left[T^{\restrict,*}_i \mid X_i\right]$.
\end{property}

An identical argument to the proof of~\Cref{prop:ey} holds for the proof of~\Cref{prop:restricted-ey}. 

The restricted IPCW estimator, $\rel^{\restrict,\ipcw}_i = \frac{\delta_i^{\restrict} T_i^{\restrict}}{\hat{G}(T_i^{\restrict} \mid X_i)}$, is also unbiased by an identical argument to the unrestricted one. 

\begin{property}[Unbiasedness of restricted IPCW estimator]\label{prop:restricted-ipcw}
Under Assumptions \ref{ass:censoring} and \ref{ass:censoring-calibration}, $\rel^{\restrict,\ipcw}_i$ is a conditionally unbiased estimator of $\E\left[T_i^{\restrict,*} \mid X_i \right]$.
\end{property}

As with the restricted case for the EY estimator, an identical argument holds for the proof of~\Cref{prop:restricted-ipcw}. 

Since $\hat{G}(t \mid X) \to 0$ in the tail, IPCW weights can grow arbitrarily large. Restricting the horizon $\tau$ bounds the maximum weight at $1/\hat{G}(\tau \mid X)$ and is a practical solution for stability. Another advantage of the restricted case of IPCW is that we only require calibration over the horizon $\tau$. In Assumption \ref{ass:censoring-calibration}, we also require that $\hat{G}(t \mid X) > 0$ to avoid division by $0$. In practice, this is not guaranteed in the unrestricted setting, but can be guaranteed in the restricted case. 

IPCW has a number of drawbacks. It assigns a weight of zero to any data point censored before $\tau$, discarding a significant portion of the dataset. In clinical settings with heavy censoring, this reduces the sample size of the evaluation metric. Although the horizon $\tau$ theoretically bounds the maximum weight, patients with a high probability of being censored still receive large inverse weights. This introduces high variance into the evaluation metric, meaning a single heavily weighted patient can entirely dominate the DCG score.

\section{Further details of baseline survival predictors}
\label{sec:appendix-baseline-models}

We present each baseline survival predictor in more detail. 
\begin{itemize}
    \item \textit{Kaplan-Meier estimator (KM)}~\citep{Kaplan58:Nonparametric} is a non-parametric model that predicts the survival function directly from empirical data. It computes the survival probability at time $t$ as the product of conditional survival probabilities at all observed event times $t_i \le t$, defined as
    $$
    \hat{S}(t) = \prod_{t_i \le t} \left( 1 - \frac{d_i}{n_i} \right)
    $$
    where $d_i$ is the number of events occurring at time $t_i$, and $n_i$ is the number of patients at risk prior to time $t_i$. The KM estimator does not use any covariates in its prediction and therefore does not explicitly satisfy the assumptions required to achieve an unbiased estimate of relevance. We still include it as a baseline as it is commonly used in the clinical setting. 
    
    \item \textit{Cox regression (Cox)}~\citep{Cox72:Regression} uses a semi-parametric model relying on the proportional hazards assumption which state that the hazard function can be expressed as 
    
    $$
    h(t \mid x) = h_0(t) \exp(\theta^\top x)
    $$
    where $h_0(t)$ is the baseline hazard function. The baseline hazard is typically fit using the Breslow estimator~\citep{Breslow75:Analysis}. The survival function, $S(t \mid x)$ is computed as
    $$
    S(t \mid x) = \exp\left( -\exp(\theta^\top x)\int_{z=0}^t h_0(z)dz\right).
    $$
    \item \textit{Accelerated Failure Time (AFT)}~\citep{Wei92:Accelerated} is a fully parametric model that assumes the effect of covariates is to accelerate (or decelerate) the time to an event. It models the logarithm of the survival time $T$ as a linear function of the covariates $x$,
    $$
    \log(T) = \theta^\top x + \sigma\epsilon
    $$
    where $\theta$ is the vector of coefficients, $\sigma$ is a scale parameter, and $\epsilon$ is an error term following an underlying assumed distribution (typically the Weibull distribution). The survival function $S(t \mid x)$ is related to the baseline survival function $S_0(t)$ by
    $$
    S(t \mid x) = S_0(t \exp(-\theta^\top x)).
    $$
    \item \textit{DeepSurv}~\citep{Katzman18:Deepsurv} is a non-linear  extension of Cox regression that uses deep neural networks to model the proportional hazards. DeepSurv replaces the linear combination of covariates with the scalar output of a deep neural network, $f_\theta(x)$. The hazard function is modeled as
    $$
    h(t \mid x) = h_0(t) \exp(f_\theta(x)).
    $$
   
    \item \textit{DeepHit}~\citep{Lee18:Deephit} is a discrete-time survival model that relaxes the proportional hazards assumption. DeepHit discretizes the time horizon into intervals and uses a deep neural network to estimate the probability of the event occurring within each interval. 
\end{itemize}

\section{Model and training configurations}
\label{sec:appendix-configs}

We detail the hyperparameters and architectural configurations for all models used in our experiments. All models are implemented in Python 3.11, leveraging common implementations from the \texttt{scikit-survival}, \texttt{lifelines}, and \texttt{pycox} libraries. We preprocess the dataset in the same way for all models. We impute missing features using median imputation, scale features using a standard scaler, and leverage a one-hot encoding of categorical features. 
All experiments are conducted on an M4 Pro processor with 24GB unified memory, and terminate within hours on this processor. We report the final hyperparameters used in experiments following tuning.

\subsection{Our bootstrapping approach}

Our approach utilizes a \textit{Gradient Boosted Decision Tree (GBDT)} implemented via the LightGBM library, optimized with a custom objective detailed in~\Cref{sec:bootstrap-model}. We present the necessary hyperparameters for reproducibility in~\Cref{tab:a-bootstrapping-hyperparameters}. Note training also uses early stopping using a validation set. 

\begin{table}[h!]
\centering
\begin{tabular}{ll}
\toprule
\textbf{Hyperparameter} & \textbf{Value} \\
\midrule
$\alpha$ & 0.3 \\
Margin $m$ & 1.0 \\
Batch size & 256 \\
Learning rate & 0.05 \\
Number of estimators & 100 \\
Number of leaves & 31 \\
Min. child samples & 40 \\
Subsample ratio & 0.8 \\
Colsample by tree & 0.8 \\
Early stopping & 50 rounds \\
\bottomrule
\end{tabular}
\caption{Bootstrapped model hyperparameters.}
\label{tab:a-bootstrapping-hyperparameters}
\end{table}

\subsection{Deep neural networks}
DeepSurv and DeepHit are implemented using the \texttt{pycox} library configured using the values in~\Cref{tab:a-nn-hyperparameters}. Training also uses early stopping based on the performance on the validation set. 

\begin{table}[h!]
\centering
\begin{tabular}{lll}
\toprule
\textbf{Hyperparameter} & \textbf{DeepSurv} & \textbf{DeepHit} \\
\midrule
Hidden layers & [64, 32] & [64, 32] \\
Activation & ReLU & ReLU \\
Dropout & 0.4 & 0.5 \\
Batch normalization & True & True \\
Optimizer & Adam & Adam \\
Learning rate & 1e-3 & 1e-3 \\
Batch size & 256 & 256 \\
Training Epochs & 75 & 75 \\
Specific Params & Weight decay: 1e-4 & Bins: 30, $\alpha$: 0.3, $\sigma$: 0.3 \\
\bottomrule
\end{tabular}
\caption{DeepSurv and DeepHit configurations.}
\label{tab:a-nn-hyperparameters}
\end{table}

\subsection{Statistical baselines}
The statistical baselines are implemented from standard libraries with $L_2$ regularization to prevent overfitting shown in~\Cref{tab:stat-hyperparameters}. 

\begin{table}[h]
\centering
\begin{tabular}{lll}
\toprule
\textbf{Model} & \textbf{Library} & \textbf{Regularization/Penalizer} \\
\midrule
Cox & \texttt{scikit-survival} & $\alpha=0.1$ \\
AFT & \texttt{lifelines} & penalizer$=0.1$ \\
KM & \texttt{lifelines} & N/A (Non-parametric) \\
\bottomrule
\end{tabular}
\caption{Statistical model parameters.}
\label{tab:stat-hyperparameters}
\end{table}

\section{Further details of the experimental setup}
\label{sec:appendix-experimental-setup}

In the artificial censoring experiment, for patients selected for censoring, we draw a censoring time $C_i \in [0, T_i^*]$ uniformly at random. We conduct 5 iterations of 5-fold cross-validation. In each iteration, the dataset is independently re-censored. We cross-fit the nuisance models within each test fold to avoid overfitting to the observed outcomes and to ensure that the nuisance estimates used for evaluation are generated independently of the data used to train the predictors.

For the full experiment on the UNOS registry, we again conduct 5 iterations of 5-fold cross-validation using the same cross-fitting setup.

\section{Omitted tables and figures}

\begin{table}[h!]
    \centering
    \small
    \begin{tabular}{cccc}
        \toprule
        \textbf{Total Patients ($N$)} & \textbf{Censored (\%)} & \textbf{Total Features} & \textbf{Categorical Features} \\
        \midrule
        60,055 & 29,707 (49.47\%) & 92 & 46 \\
        \bottomrule
    \end{tabular}
    \caption{Summary of UNOS dataset characteristics and feature composition.}
    \label{tab:cohort_summary}
\end{table}

\begin{table}[h!]
    \centering
    \begin{tabular}{lccc}
        \toprule
        \textbf{Metric (Years)} & \textbf{Entire Cohort} & \textbf{Observed} & \textbf{Censored} \\
        \midrule
        Median & 5.60  & 5.71  & 5.25  \\
        75th Percentile & 10.83 & 11.46 & 10.01 \\
        90th Percentile & 16.32 & 17.01 & 15.82 \\
        95th Percentile & 20.01 & 20.37 & 19.64 \\
        99th Percentile & 26.76 & 26.42 & 27.01 \\
        Maximum & 36.02 & 35.25 & 36.02 \\
        \bottomrule
    \end{tabular}
    \caption{Distribution of follow-up times in years across the overall cohort, observed events, and censored patients.}
    \label{tab:followup_distribution}
\end{table}

\begin{figure}[h!]
    \centering
    \includegraphics[width=0.5\linewidth]{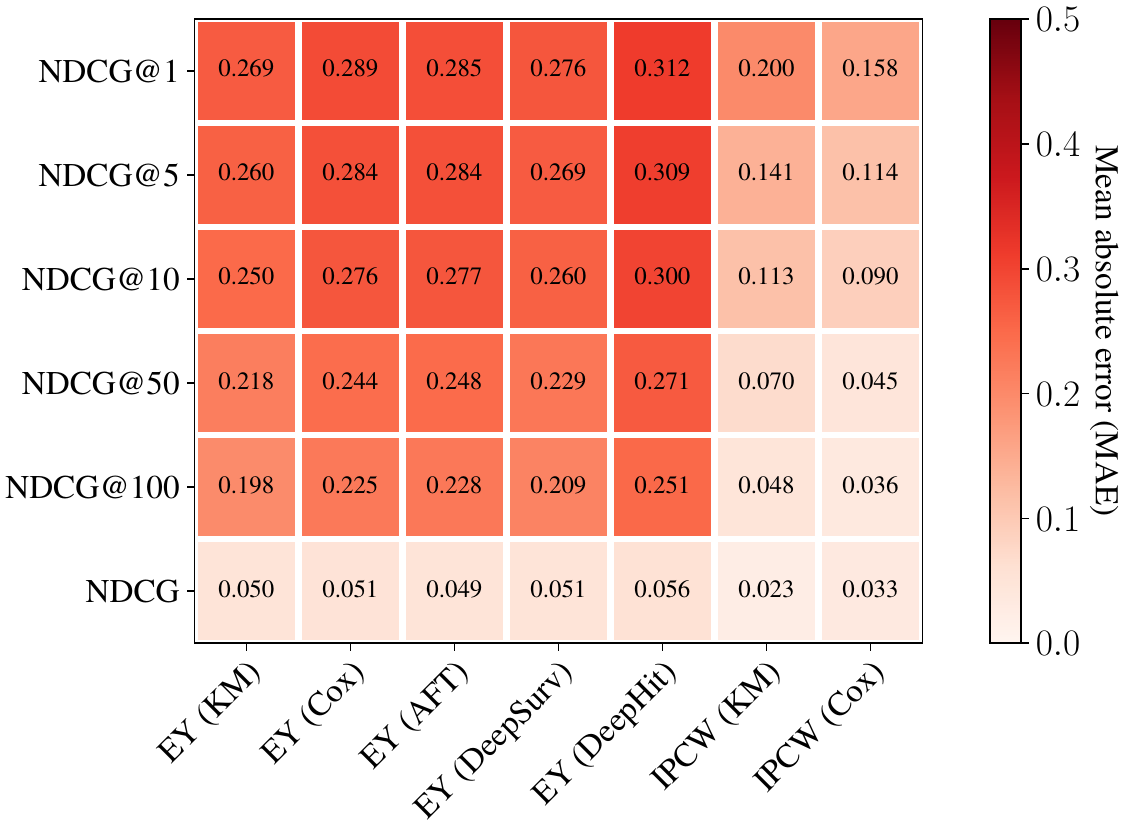}
    \caption{Mean absolute error in NDCG estimation.}
    \label{fig:mae-heatmap}
\end{figure}

\begin{figure*}[h!]
    \centering
    \begin{subfigure}[b]{0.32\textwidth}
        \centering
        \includegraphics[width=\linewidth]{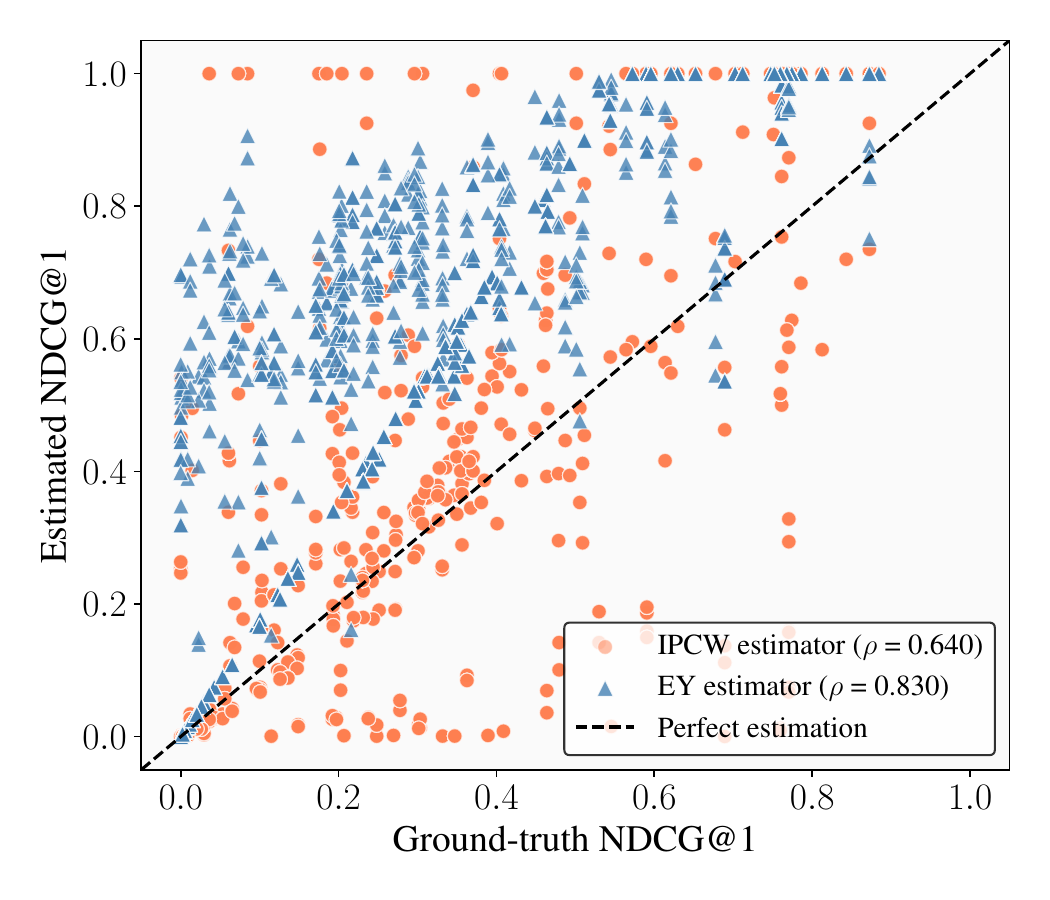}
        \caption{Estimation of NDCG@1.}
        \label{fig:ndcg@1-est}
    \end{subfigure}
    \hfill
    \begin{subfigure}[b]{0.32\textwidth}
        \centering
        \includegraphics[width=\linewidth]{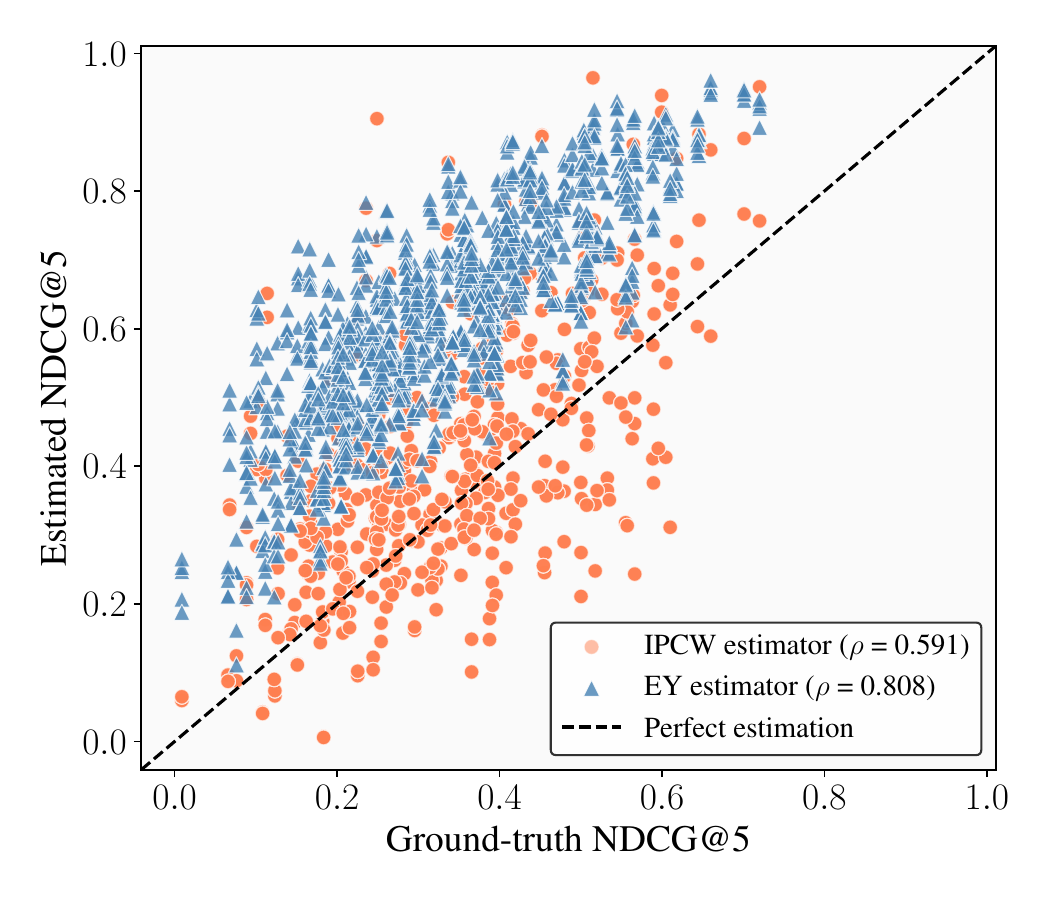}
        \caption{Estimation of NDCG@5.}
        \label{fig:ndcg@5-est}
    \end{subfigure}
    \hfill
    \begin{subfigure}[b]{0.32\textwidth}
        \centering
        \includegraphics[width=\linewidth]{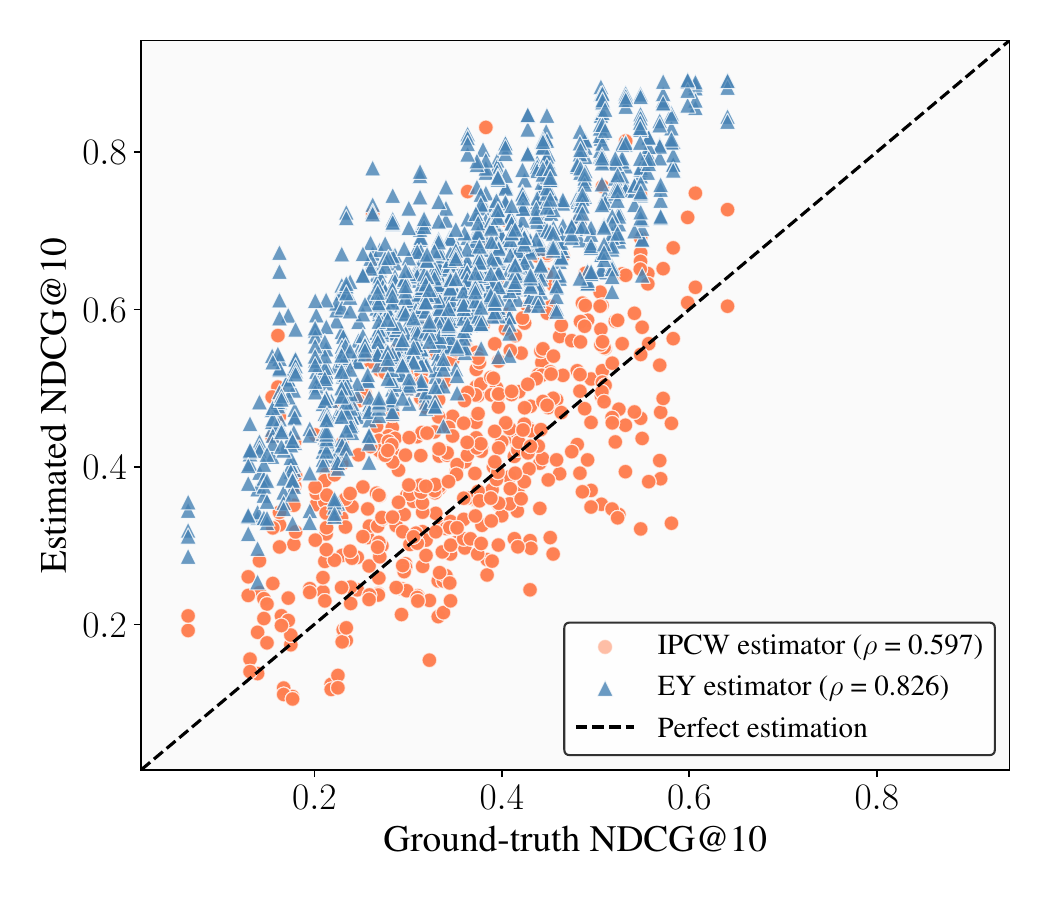}
        \caption{Estimation of NDCG@10.}
        \label{fig:ndcg@10-est}
    \end{subfigure}
    \\
    \begin{subfigure}[b]{0.32\textwidth}
        \centering
        \includegraphics[width=\linewidth]{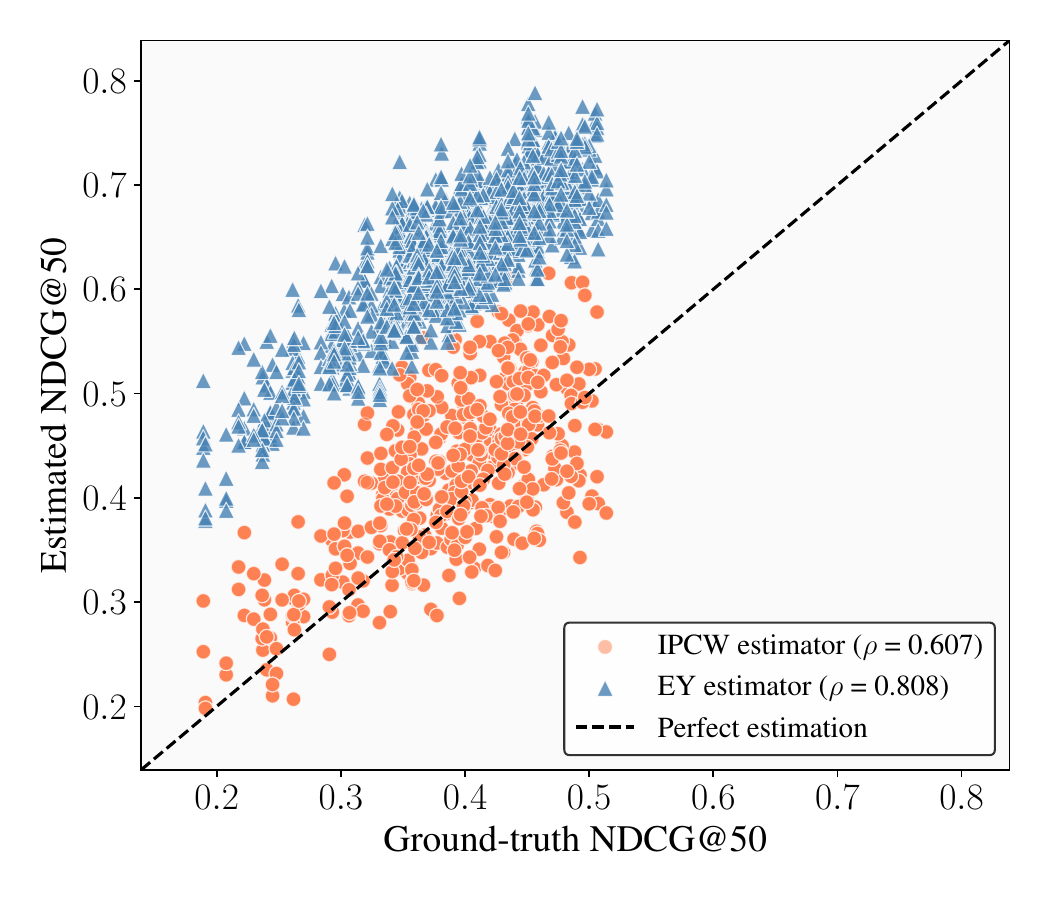}
        \caption{Estimation of NDCG@50.}
        \label{fig:ndcg@50-est}
    \end{subfigure}
    \hfill
    \begin{subfigure}[b]{0.32\textwidth}
        \centering
        \includegraphics[width=\linewidth]{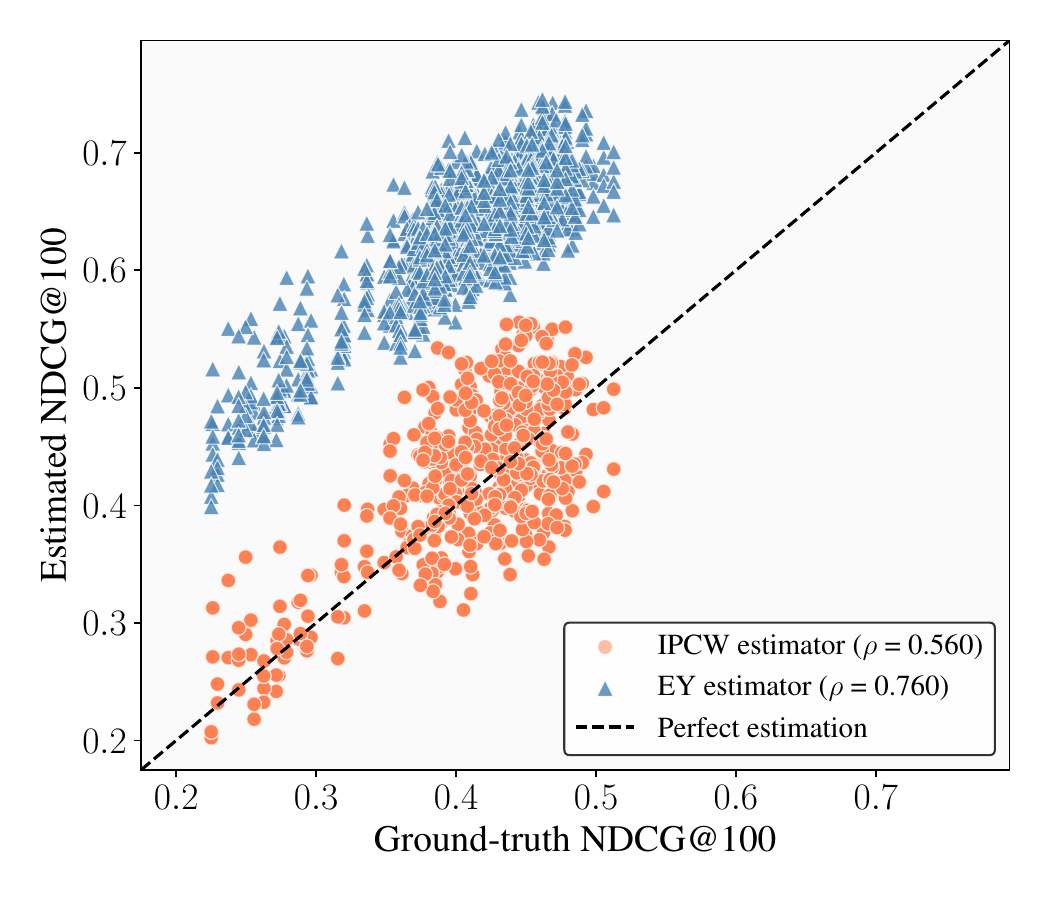}
        \caption{Estimation of NDCG@100.}
        \label{fig:ndcg@100-est}
    \end{subfigure}
    \hfill
    \begin{subfigure}[b]{0.32\textwidth}
        \centering
        \includegraphics[width=\linewidth]{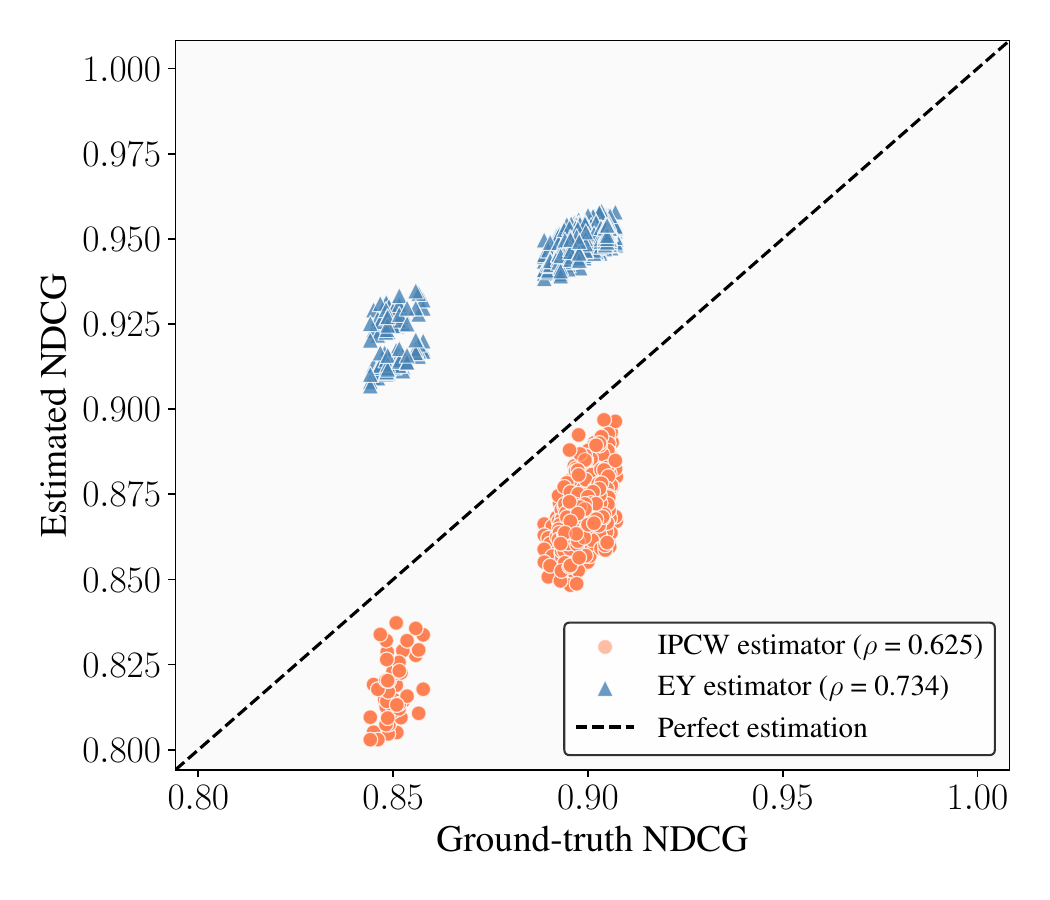}
        \caption{Estimation of NDCG.}
        \label{fig:ndcg-est}
    \end{subfigure}
    
    \caption{Ground-truth versus estimated NDCG@$k$. Each data point is an individual estimate for a model and estimator. We report the Spearman rank correlation coefficient, $\rho$, for each estimator.}
    \label{fig:ndcg-scatter}
\end{figure*}

\begin{table}[h!]
\centering
\small
\resizebox{\textwidth}{!}{
\begin{tabular}{llcccccc}
\toprule
\textbf{Method} & \textbf{Nuisance model} & NDCG@1 & NDCG@5 & NDCG@10 & NDCG@50 & NDCG@100 & NDCG \\
\midrule
IPCW & KM & 77.8\% & 83.1\% & 78.7\% & 81.3\% & 85.8\% & 91.1\% \\
IPCW & Cox & 80.9\% & 73.8\% & 76.9\% & 72.9\% & 74.7\% & 83.1\% \\
IPCW & Ensemble (Avg) & 80.9\% & 79.6\% & 77.8\% & 79.1\% & 80.0\% & 87.1\% \\
\midrule
EY & KM & 83.6\% & 83.2\% & 84.5\% & 85.5\% & 84.6\% & 86.1\% \\
EY & Cox & 82.3\% & 81.1\% & 81.7\% & 83.4\% & 81.7\% & 83.7\% \\
EY & AFT & 83.2\% & 82.6\% & 83.5\% & 84.6\% & 83.7\% & 85.7\% \\
EY & DeepSurv & 83.1\% & 82.2\% & 82.8\% & 84.1\% & 83.2\% & 85.4\% \\
EY & DeepHit & 81.7\% & 79.3\% & 80.7\% & 82.5\% & 81.5\% & 82.3\% \\
EY & Ensemble (Avg) & 80.9\% & 83.1\% & 89.3\% & 91.1\% & 92.9\% & 91.1\% \\
\bottomrule
\end{tabular}
}
\caption{Model pairwise accuracy by estimator.}
\label{tab:estimators_artificial_censoring}
\end{table}

\begin{table}[h!]
\centering
\resizebox{\textwidth}{!}{
\begin{tabular}{ll ccccc}
\toprule
\textbf{Predictor} & \textbf{Framework} & \multicolumn{5}{c}{\textbf{EY NDCG@$k$}} \\
\cmidrule(lr){3-7}
 & & \textbf{1} & \textbf{5} & \textbf{10} & \textbf{50} & \textbf{100} \\
\midrule
KM & Standard & 0.647 $\pm 0.288$ & 0.610 $\pm 0.158$ & 0.604 $\pm 0.094$ & 0.614 $\pm 0.040$ & 0.612 $\pm 0.027$ \\
 & + Ours & \textbf{0.888 $\pm 0.184^{**}$} & \textbf{0.833 $\pm 0.108^{**}$} & \textbf{0.824 $\pm 0.075^{**}$} & \textbf{0.790 $\pm 0.036^{**}$} & \textbf{0.773 $\pm 0.029^{**}$} \\
\hline
Cox & Standard & 0.699 $\pm 0.271$ & 0.729 $\pm 0.100$ & 0.729 $\pm 0.065$ & 0.729 $\pm 0.036$ & 0.730 $\pm 0.022$ \\
 & + Ours & \textbf{0.834 $\pm 0.066^{**}$} & \textbf{0.800 $\pm 0.056^{**}$} & \textbf{0.801 $\pm 0.045^{**}$} & \textbf{0.776 $\pm 0.033^{**}$} & \textbf{0.769 $\pm 0.024^{**}$} \\
\hline
AFT & Standard & 0.600 $\pm 0.312$ & 0.682 $\pm 0.153$ & 0.707 $\pm 0.105$ & 0.732 $\pm 0.042$ & 0.718 $\pm 0.028$ \\
 & + Ours & \textbf{0.810 $\pm 0.194^{**}$} & \textbf{0.778 $\pm 0.105^{**}$} & \textbf{0.767 $\pm 0.081^{**}$} & \textbf{0.760 $\pm 0.032^{**}$} & \textbf{0.752 $\pm 0.027^{**}$} \\
\hline
DeepSurv & Standard & 0.600 $\pm 0.333$ & 0.622 $\pm 0.142$ & 0.631 $\pm 0.106$ & 0.676 $\pm 0.047$ & 0.690 $\pm 0.034$ \\
 & + Ours & \textbf{0.759 $\pm 0.169$} & \textbf{0.756 $\pm 0.094^{**}$} & \textbf{0.757 $\pm 0.065^{**}$} & \textbf{0.757 $\pm 0.039^{**}$} & \textbf{0.752 $\pm 0.032^{**}$} \\
\hline
DeepHit & Standard & 0.649 $\pm 0.284$ & 0.650 $\pm 0.130$ & 0.647 $\pm 0.088$ & 0.669 $\pm 0.043$ & 0.685 $\pm 0.039$ \\
 & + Ours & \textbf{0.730 $\pm 0.306^{*}$} & \textbf{0.769 $\pm 0.102^{**}$} & \textbf{0.759 $\pm 0.093^{**}$} & \textbf{0.763 $\pm 0.046^{**}$} & \textbf{0.755 $\pm 0.035^{**}$} \\
\bottomrule
\multicolumn{7}{l}{{\small $^{*} p < 0.05, ^{**} p < 0.01$ (Wilcoxon signed-rank test vs standard model)}} \\
\end{tabular}}
\caption{Performance of models from the EY NDCG estimator. We report the average over all nuisance models for each estimator.}
\end{table}

\begin{table}[h!]
\centering
\resizebox{\textwidth}{!}{
\begin{tabular}{ll ccccc}
\toprule
\textbf{Predictor} & \textbf{Framework} & \multicolumn{5}{c}{\textbf{IPCW NDCG@$k$}} \\
\cmidrule(lr){3-7}
 & & \textbf{1} & \textbf{5} & \textbf{10} & \textbf{50} & \textbf{100} \\
\midrule
KM & Standard & 0.430 $\pm 0.388$ & 0.350 $\pm 0.189$ & 0.322 $\pm 0.138$ & 0.296 $\pm 0.080$ & 0.287 $\pm 0.064$ \\
 & + Ours & \textbf{0.758 $\pm 0.523^{**}$} & \textbf{0.688 $\pm 0.284^{**}$} & \textbf{0.619 $\pm 0.221^{**}$} & \textbf{0.539 $\pm 0.111^{**}$} & \textbf{0.516 $\pm 0.078^{**}$} \\
\hline
Cox & Standard & 0.155 $\pm 0.414$ & 0.198 $\pm 0.202$ & 0.177 $\pm 0.131$ & 0.213 $\pm 0.070$ & 0.244 $\pm 0.056$ \\
 & + Ours & \textbf{0.370 $\pm 0.468^{**}$} & \textbf{0.427 $\pm 0.260^{**}$} & \textbf{0.412 $\pm 0.212^{**}$} & \textbf{0.323 $\pm 0.089^{**}$} & \textbf{0.332 $\pm 0.074^{**}$} \\
\hline
AFT & Standard & 0.068 $\pm 0.129$ & 0.081 $\pm 0.068$ & 0.087 $\pm 0.052$ & 0.094 $\pm 0.033$ & 0.107 $\pm 0.032$ \\
 & + Ours & \textbf{0.128 $\pm 0.179$} & \textbf{0.178 $\pm 0.130^{**}$} & \textbf{0.213 $\pm 0.113^{**}$} & \textbf{0.212 $\pm 0.062^{**}$} & \textbf{0.223 $\pm 0.048^{**}$} \\
\hline
DeepSurv & Standard & 0.183 $\pm 0.302$ & 0.112 $\pm 0.113$ & 0.148 $\pm 0.099$ & 0.183 $\pm 0.062$ & 0.223 $\pm 0.057$ \\
 & + Ours & \textbf{0.269 $\pm 0.521$} & \textbf{0.355 $\pm 0.334^{**}$} & \textbf{0.322 $\pm 0.244^{**}$} & \textbf{0.292 $\pm 0.124^{**}$} & \textbf{0.302 $\pm 0.088^{**}$} \\
\hline
DeepHit & Standard & 0.192 $\pm 0.458$ & 0.175 $\pm 0.178$ & 0.153 $\pm 0.120$ & 0.239 $\pm 0.071$ & 0.292 $\pm 0.064$ \\
 & + Ours & \textbf{0.601 $\pm 0.659^{**}$} & \textbf{0.487 $\pm 0.257^{**}$} & \textbf{0.439 $\pm 0.205^{**}$} & \textbf{0.421 $\pm 0.110^{**}$} & \textbf{0.421 $\pm 0.079^{**}$} \\
\bottomrule
\multicolumn{7}{l}{{\small $^{*} p < 0.05, ^{**} p < 0.01$ (Wilcoxon signed-rank test vs standard model)}} \\
\end{tabular}}
\caption{Performance of models from the IPCW NDCG estimator. We report the average over all nuisance models for each estimator.}
\end{table}

\end{document}